\pdfoutput=1

\documentclass[11pt]{article}

\usepackage[]{ACL2023}

\usepackage{times}
\usepackage{latexsym,bm}
\usepackage{amsthm,thmtools, thm-restate}

\usepackage[T1]{fontenc}

\usepackage[utf8]{inputenc}

\usepackage{microtype}

\usepackage{inconsolata}

\usepackage{amsmath}
\usepackage{amssymb}
\usepackage{mathtools}
\usepackage{amsthm}
\usepackage{thmtools}
\usepackage{thm-restate}
\usepackage{graphicx}
\usepackage{placeins}
\usepackage{amsfonts}
\usepackage{paralist}
\usepackage{subcaption}
\usepackage{pifont}

\usepackage{multirow}
\usepackage{wasysym}
\theoremstyle{plain}
\usepackage{siunitx}
\usepackage{paralist}
\usepackage{fancyhdr}

\theoremstyle{definition}

\theoremstyle{remark}

\newcommand{\E}{\operatorname*{\mathbb E}}
\newcommand{\bfY}{{\mathbf Y}}
\newcommand{\rmY}{{\mathrm Y}}

\title{$f$-Divergence Minimization for Sequence-Level Knowledge Distillation}

\author{
Yuqiao Wen$^{1,*}$, Zichao Li$^{2,*}$, Wenyu Du$^3$, Lili Mou$^{1,4}$ \vspace{0.2cm} \\
\normalsize $^1$Dept.~Computing Science \& Alberta Machine Intelligence Institute (Amii), 
University of Alberta\\
\normalsize$^2$Mila, McGill University \quad\quad $^3$The University of Hong Kong \\
\normalsize$^4$Canada CIFAR AI Chair, Amii \quad\quad\ \ \ \quad $^*$Equal contribution \\
\normalsize\texttt{\url{yq.when@gmail.com}, \url{zichao.li@mila.quebec}} \\
\normalsize \texttt{\url{wenyudu@yahoo.com}, \url{doublepower.mou@gmail.com}} \\
}

\begin{document}

\maketitle

\newcommand{\fdistill}{\textsc{$f$-distill}}
\newcommand\scalemath[2]{\scalebox{#1}{\mbox{\ensuremath{\displaystyle #2}}}}

\begin{abstract}

Knowledge distillation (KD) is the process of transferring knowledge from a large model to a small one.
It has gained increasing attention in the natural language processing community, driven by the demands of compressing ever-growing language models.
In this work, we propose an \textsc{$f$-distill} framework, which formulates sequence-level knowledge distillation as minimizing a generalized $f$-divergence function.
We propose four distilling variants under our framework and show that existing SeqKD and ENGINE approaches are approximations of our \fdistill{} methods.
We further derive step-wise decomposition for our \fdistill{}, reducing intractable sequence-level divergence to word-level losses that can be computed in a tractable manner.
Experiments across four datasets show that our methods outperform existing KD approaches, and that our symmetric distilling losses can better force the student to learn from the teacher distribution.\footnote{
Our code is available at \url{https://github.com/MANGA-UOFA/fdistill}
}

\end{abstract}

\section{Introduction} \label{sec:intro}
\allowdisplaybreaks

\renewcommand{\headrulewidth}{0pt}
\fancyhf{}
\cfoot{In \textit{ACL'23}, pages 10817–10834, with additional notes from the conference.}
\thispagestyle{fancy}

Increasingly large language models have continued to achieve state-of-the-art performance across various natural language generation tasks, such as data-to-text generation~\citep{lebret-etal-2016-neural,li-liang-2021-prefix}, summarization~\citep{paulus2018deep,pmlr-v119-zhang20ae}, and dialogue generation~\citep{dialogue-reinforcement,zhang-etal-2020-dialogpt}.
However, super-large language models are inaccessible to most users and researchers due to their prohibitively large model size, emphasizing the importance of high-performing, parameter-efficient small neural models.

A widely used approach to training small models is \textit{knowledge distillation}~\cite[KD,][]{hinton2015distilling}, where the small model (known as the \textit{student}) learns the knowledge from a much larger model (known as the \textit{teacher}). 
KD has shown great success in helping smaller models achieve competitive performance across a wide range of applications~\citep{sun2019pkd,jiao-etal-2020-tinybert,shleifer2020pre}.

Existing KD approaches can be categorized into two main branches: representation matching and distribution matching.
The former aims to imitate the teacher's real-valued intermediate-layer representations, say, with mean squared error~\citep{sun2019pkd,jiao-etal-2020-tinybert}.
Our work focuses on the latter, distribution matching, where the student model learns the teacher's predictive distribution.
\citet{hinton2015distilling} minimize the cross-entropy loss against the teacher-predicted soft labels, which is equivalent to minimizing the Kullback--Leibler (KL) divergence between the teacher and student.
\citet{kim-rush-2016-sequence} propose SeqKD, arguing that KL divergence should be minimized at the sequence level for language models. However, such an approach tends to learn an overly smooth student distribution to cover the entire support  of the teacher distribution due to the asymmetric nature of the KL divergence. This is often known as the \textit{mode-averaging} problem (Figure~\ref{fig:modes}a).

\citet{tu-etal-2020-engine} propose ENGINE, a non-autoregressive translation model that minimizes the energy function defined by the teacher's output distribution. It can be shown that their objective is related to minimizing the reverse KL between the teacher and student (see Section~\ref{subsec:fdistill}). This, on the other hand, results in the \textit{mode-collapsing} problem, where the student model is overly concentrated on certain high-probability regions of the teacher distribution (Figure~\ref{fig:modes}b).

\begin{figure}[!t]
    \centering
    \includegraphics[width=\columnwidth]{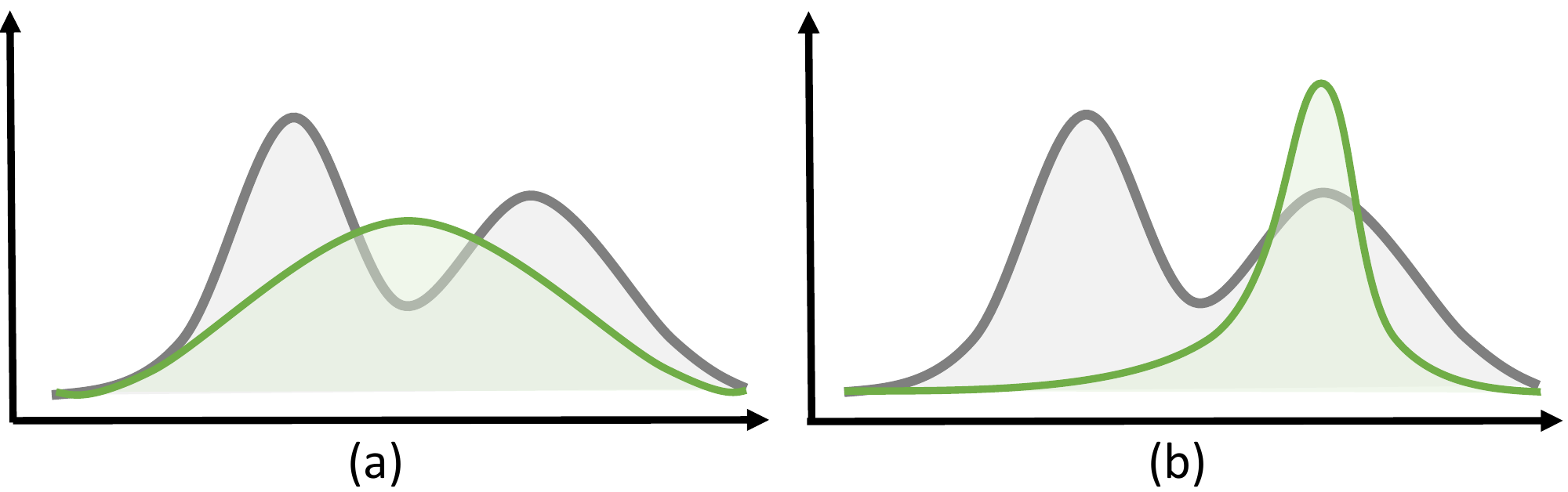}
    \caption{ Typical phenomena of (a) KL distillation and (b) reverse KL distillation.
    Gray curve: teacher distribution. Green curve: student distribution. }
\label{fig:modes}
\vspace*{-1em}
\end{figure}

In this paper, we address  knowledge distillation for text generation tasks, and propose \textsc{$f$-distill}, a unified framework that formulates sequence-level knowledge distillation as minimizing $f$-divergence functions.
Existing SeqKD~\citep{kim-rush-2016-sequence} and ENGINE~\citep{tu-etal-2020-engine} methods are  approximations of KL and reverse KL distillations under the \textsc{$f$-distill} framework. Further, our formulation naturally leads to Jensen--Shannon (JS) divergence and total variation distance (TVD) distillations, where the divergence measures are symmetric in teacher and student distributions. This forces the student to learn the teacher's distribution better, alleviating mode averaging and collapsing problems.

We further develop efficient algorithms for our \textsc{$f$-distill} approach. First, we show that sequence-level \textsc{$f$}-divergence can be decomposed step by step either exactly or as an upper bound. Second, we propose to sample from the teacher model in an offline manner, mitigating the additional training cost of symmetric divergence measures (namely, JS and TVD).

We evaluated our approach on four datasets: DART for data-to-text generation~\citep{nan-etal-2021-dart}, XSum for summarization~\citep{narayan-etal-2018-xsum}, WMT16 EN-RO for machine translation~\citep{bojar-etal-2016-wmt16}, and Commonsense Dialogue~\citep{zhou-etal-2021-commonsense}.
Experiments show that our proposed \textsc{$f$-distill} variants consistently outperform existing distribution-matching KD methods, allowing \textsc{$f$-distill} to achieve an add-on performance improvement when combined with representation-matching KD methods.
Further, results show that our symmetric distilling losses outperform asymmetric ones, confirming that extreme mode averaging or collapsing is not ideal.

To sum up, our contributions are three-fold:
\begin{compactenum}
    \item We propose \textsc{$f$-distill}, a novel distilling framework that generalizes KL distillation and balances mode averaging and collapsing;
    \item We derive step-wise decomposition and propose an offline sampling method to efficiently compute sequence-level $f$-divergences; and 
    \item We provide detailed experimental analysis across four text generation datasets to show the effectiveness of our approach.
\end{compactenum}

\section{Approach}

In this section, we first review classic knowledge distilling (KD) algorithms and analyze their drawbacks.
Then, we propose \textsc{$f$-distill}, a generalized distilling framework for sequence-level distillation.

\subsection{Classic KD and Its Drawbacks} \label{subsec:classic-kd}
In classic KD, the KL divergence is often used to train the student model to match the teacher's distribution~\citep{hinton2015distilling}.
For autoregressive text generation, this is decomposed into a step-wise KL divergence:
\begin{align}\label{eq:classicKD}
    J_{\text{KD}} = - \sum_{t=1}^{|\mathbf y|} \sum_{\mathrm Y_t\in V} p(\mathrm Y_t| \mathbf y_{<t}) \log q_{\theta}(\mathrm Y_t| \mathbf y_{<t})
\end{align}
where $\mathbf y = \mathrm y_1\cdots\mathrm y_T$ is the ground-truth sequence and  $V$ is the vocabulary.
$p$ and $q_\theta$ are the predicted distributions of the teacher and student, respectively; they can be additionally conditioned on an input sequence $\mathbf x$, which is omitted here for simplicity. In Eqn.~\eqref{eq:classicKD}, we present the loss by a cross-entropy term, which only differs from the KL divergence $D_\text{KL}(p\|q_\theta)$ by a constant.  

\citet{kim-rush-2016-sequence} propose SeqKD and minimize cross-entropy loss at the sequence level as
\begin{align}
    J_{\text{SeqKD}} &= \mathbb E_{\mathbf Y \sim p} [- \log q_\theta (\mathbf Y)]
\end{align}
In practice, the expectation over the sentence space is intractable, so they approximate it with a hard sequence $\mathbf y$ generated by beam search on the teacher model. Their loss is
\begin{align}
    \hat{J}_{\text{SeqKD}} = - \sum_{t=1}^{|\mathbf y|} \log q_{\theta}(\mathrm y_t|\mathbf y_{<t})
    \label{eq:seqkd-approx}
\end{align}

However, KL-based losses may cause the student model to learn an overly smooth function.
This can be seen in Eqn.~\eqref{eq:seqkd-approx}, where the loss term $-\log q_{\theta}(\mathrm y_t| \mathbf y_{<t})$ goes to infinity when the student assigns a low probability to a teacher-generated token.
As a result, minimizing KL forces the student model to spread its probability mass widely over the vocabulary.
When the student has a limited model capacity, this further leads to the mode-averaging problem, where the learned distribution may not capture any mode of the teacher distribution, as shown in Figure~\ref{fig:modes}a.

\subsection{Our Proposed \textsc{$f$-distill} Framework} \label{subsec:fdistill}

To this end, we propose a generalized \mbox{\textsc{$f$-distill}} framework, a family of distilling methods based on $f$-divergence functions~\citep{classic-fdiv,ieee-fdiv}.

Formally, the $f$-divergence of two distributions is defined as
\begin{align}
    D_f(p(t)\|q(t)) = \sum_{t} q(t)\, f\!\left(\frac{p(t)} {q(t)}\right) &
\end{align}
where $f: (0, \infty) \rightarrow \mathbb R $ is a convex function such that $f(1) = 0$.
Table~\ref{tab:f-div-map} summarizes common divergence functions. 

In the rest of this subsection, we will first present Kullback--Leibler (KL) and reverse KL (RKL) distilling methods, which are closely related to previous work~\cite{kim-rush-2016-sequence,tu-etal-2020-engine}. Then, we will propose Jensen--Shannon (JS) and total variation distance (TVD) distillations; they are based on symmetric $f$-divergence functions, and are able to force the student to better learn from the teacher distribution.

\begin{table}[!t]
\centering
\resizebox{\columnwidth}{!}{
\begin{tabular}{l|l}
\hline
Divergence & $f(t)$ \\ \hline
Kullback--Leibler (KL)                & $t \log t$           \\
Reverse KL (RKL)                 & $- \log t$           \\
Jensen--Shannon (JS)                  & $-(t+1) \log (\frac{t+1}{2}) + t \log t$                    \\
Total variation distance (TVD)       & $\frac12 |t-1|$      \\ \hline
\end{tabular}
}
\caption{Common divergence functions and their corresponding choices of $f$.}
\label{tab:f-div-map}
\end{table}

\textbf{Kullback--Leibler (KL) distillation.} Recall that we denote the teacher distribution by $p$ and the student distribution by $q_\theta$.
Using the common KL divergence leads to the standard distilling objective
\begin{align}\label{eq:KL1}
    & J_{\text{KL}} = D_\text{KL}(p\|q_\theta)=\mathbb E_{\mathbf Y\sim p} \left[ \log \frac{p(\mathbf Y)}{q_\theta (\mathbf Y)} \right]  \\
    &\approx - \sum_{t=1}^{|\mathbf y|} \sum_{\mathrm Y_t \in V} p(\mathrm Y_t| \mathbf{y}_{<t}) \log q_{\theta}(\mathrm Y_t| \mathbf{y}_{<t}) + \text{const}
\end{align}
where $\mathbf y$ is sampled\footnote{In our method, the expectation \eqref{eq:KL1} is approximated by one Monte Carlo-sampled sequence. We denote a sampled sequence by a lower letter $\mathbf y$.} from the teacher distribution $p$.
Here, the constant is the entropy of $p$, which can be ignored as it does not involve the student parameters.

Similar to SeqKD, such KL distillation may also suffer from the mode-averaging problem and learn an overly smooth distribution, because $q_\theta$ is in the denominator in \eqref{eq:KL1}.

However, our KL distillation differs from \mbox{SeqKD} in that we adopt soft labels from the teacher model, i.e., keeping the entire distribution of $p(\mathrm Y_t| {\mathbf{y}}_{<t})$, whereas SeqKD uses a certain decoded sequence~${\mathbf y}$ as shown in Eqn.~\eqref{eq:seqkd-approx}. Experiments will show that our soft labels provide more information than hard SeqKD in sequence-level distilling tasks, which is consistent with early evidence~\citep{bucilua2006model, hinton2015distilling}.

\textbf{Reverse KL (RKL) distillation.}
We propose RKL distillation, which can potentially address the mode-averaging problem:
\begin{align}
    & J_{\text{RKL}} =D_\text{KL}(q_\theta\|p)= \mathbb E_{\mathbf Y'\sim q_\theta} \left[\log\frac{q_\theta (\mathbf Y')}{p(\mathbf Y')}\right] \nonumber \\
    &\approx \sum_{t=1}^{|\mathbf y'|} \sum_{\mathrm Y'_t \in V} \Big[ q_\theta(\mathrm Y'_t | \mathbf y'_{<t}) \log q_\theta(\mathrm Y'_t| \mathbf y'_{<t}) \nonumber \\
    &\quad\quad - q_\theta(\mathrm Y'_t| \mathbf y'_{<t}) \log p(\rmY'_t | \mathbf y_{<t}') \Big]
    \label{eq:rkl}
\end{align}
where $\mathbf y'$ is sampled from the student distribution.
In other words, the loss can be decomposed into the negative log probability of the teacher's predicted probability plus the entropy of the student. 

RKL does not suffer from mode averaging because the student distribution~$q_\theta$ goes to the numerator and does not have to cover the teacher distribution. Also, the entropy term in \eqref{eq:rkl} penalizes the student for learning a wide-spreading distribution, further mitigating the mode-averaging problem.

However, RKL distillation has the opposite problem, known as mode collapsing, where the student only learns one or a few modes of the teacher distribution. This is because the RKL loss would be large, if $q_\theta(\mathbf Y')$ is high but $p(\mathbf Y')$ is low for some $\mathbf Y'$. As a result, the student tends to overly concentrate its probability mass on certain high-probability regions of the teacher model, which may not be ideal either (Figure~\ref{fig:modes}b).

RKL distillation is related to the ENGINE distilling approach~\citep{tu-etal-2020-engine}, which was originally designed to minimize the energy function defined by the teacher model.
In particular, the ENGINE objective approximates RKL less the student entropy: $J_{\text{ENGINE}} = \mathbb E_{\mathbf Y \sim q_\theta}[- \log p(\mathbf Y)]$.
Therefore, ENGINE also suffers from the mode-collapsing problem, resembling RKL distillation.

\bigskip\textbf{Remarks.} KL and RKL have the mode-averaging or mode-collapsing problem, because $D_{\text{KL}}(\cdot\|\cdot)$ is asymmetric in its two arguments, requiring the second distribution to cover the support of the first.
In the following, we will propose two \textsc{$f$-distill} variants based on symmetric divergence functions to seek a balance between these two extremes.

\bigskip\textbf{Jenson--Shannon (JS) distillation.}
Our proposed JS distillation minimizes the JS divergence, which measures the difference between two distributions and their average.
We derive the step-wise decomposition of the sequence-level JS loss:
\begin{align}
    &J_{\text{JS}} = \frac12 \E_{\mathbf Y \sim p} \left[ \log \tfrac{p(\mathbf Y)}{m(\mathbf Y)} \right] \nonumber + \frac12 \E_{\mathbf Y' \sim q_\theta} \left[ \log \tfrac{q_\theta (\mathbf Y')}{m(\mathbf Y')} \right] \nonumber \\
    &\approx \frac12 \sum_{t=1}^{|\mathbf y|} \sum_{\mathrm Y_t \in V} - p(\mathrm Y_t| \mathbf y_{<t}) \log ( m(\mathrm Y_t| \mathbf y_{<t}) ) \nonumber \\
    &+ \frac12 \sum_{t=1}^{|\mathbf y'|} \sum_{\mathrm Y'_t \in V} \left[q_\theta(\mathrm Y'_t| \mathbf y'_{<t}) \log(q_\theta(\mathrm Y'_t| \mathbf y'_{<t}) \right. \nonumber \\
    &\left. -q_\theta(\mathrm Y'_t| \mathbf y'_{<t}) \log ( m(\mathrm Y'_t| \mathbf y'_{<t}) ) \right] + \text{const}\label{eq:JS-main}
\end{align}
where $\mathbf y$ and $\mathbf y'$ are sampled from the teacher's and student's distributions, which are compared with their average $m(\cdot)=\frac12 p(\cdot) + \frac12q_\theta(\cdot)$. Appendix~\ref{apdx:proof} provides the proof of this decomposition, and Subsection~\ref{subsec:imple-consid} presents an efficient approximation by avoiding on-the-fly sampling from the teacher.

\textbf{Total variation distance (TVD) distillation.}
Our \textsc{$f$-distill} gives rise to another novel distilling variant based on the total variation distance
\begin{align}
    J_{\text{TVD}} = \frac12 \sum_{\mathbf Y} |q_\theta(\mathbf Y) - p(\mathbf Y)|
\end{align}
Unlike JS divergence, TVD measures the $\ell^1$ norm between two distributions, and therefore does not have the $\log$ operator, making the gradient more stable than JS distillation.

We would like to decompose the sequence-level TVD step by step due to the intractable summation over the sentence space. However, TVD decomposition is non-trivial, and we show in Appendix~\ref{apdx:proof} that the sequence-level TVD is upper bounded by step-wise terms, being our objective to minimize:

\begin{align} 
    & J_{\text{TVD}} = \frac12 \sum_{\mathbf Y} |q_\theta(\mathbf Y) - p(\mathbf Y)| \nonumber \\ 
    & \scalemath{0.95} {\le \frac14 \E_{\mathbf Y \sim p} \Bigg[ \sum_{t=1}^{|\mathbf Y|} \sum_{\mathrm Y_t \in V} |q_\theta(\mathrm Y_t|\mathbf{Y}_{<t}) - p(\mathrm Y_t|\mathbf{Y}_{<t})| \Bigg] } \nonumber \\
    & \scalemath{0.95} { + \frac14 \E_{\mathbf{Y}' \sim q_\theta} \Bigg[ \sum_{t=1}^{|\mathbf Y'|} \sum_{\mathrm Y_t' \in V} |q_\theta(\mathrm Y_t'|\mathbf{Y}'_{<t}) - p(\mathrm Y_t'|\mathbf{Y}'_{<t})| \Bigg] } \nonumber \\
    & \approx \frac14 \sum_{t=1}^{|\mathbf y|} \sum_{\mathrm Y_t \in V} |q_\theta(\mathrm Y_t|\mathbf y_{<t}) - p(\mathrm Y_t|\mathbf y_{<t})| \nonumber \\
    & + \frac14 \sum_{t=1}^{|\mathbf y'|} \sum_{\mathrm Y'_t \in V} |q_\theta(\mathrm Y'_t|\mathbf{y}'_{<t}) - p(\mathrm Y'_t|\mathbf{y}'_{<t})| \label{eq:TVD}
\end{align}
where ${\mathbf y}$ and ${\mathbf y}'$ are again sampled from the teacher and student models, respectively.

\bigskip
\textbf{Summary.} In this part, we have described our proposed \textsc{$f$-distill} framework with four variants based on different $f$-divergence functions. We have also presented their step-wise decompositions, whose justification is summarized by the following theorem, proved in Appendix~\ref{apdx:proof}.

\begin{restatable}{thm}{decomposition}
(a) The sequence-level KL, RKL, and JS divergences can be decomposed exactly into step-wise terms. (b) The sequence-level TVD can be upper bounded by step-wise terms.
\label{thm:decomposition}
\end{restatable}

\subsection{Implementation Considerations} \label{subsec:imple-consid}
\textbf{Efficient approximation.}
Symmetric distilling losses (i.e., JS and TVD) are slow to compute, because they require sampling from both teacher and student models during training.

We propose to mitigate this by offline sampling for the teacher model to improve training efficiency. Specifically, we obtain teacher samples, i.e., $\mathbf y$ in Eqns.~\eqref{eq:JS-main} and~\eqref{eq:TVD}, beforehand and keep them fixed during training.
This is feasible because the teacher model is unchanged and hence does not require multiple inferences, whereas the student model is continuously updated and thus requires inference in an online fashion.
Experiments show that such a treatment  significantly improves the training efficiency for both JS and TVD distillations.

\textbf{Pre-distillation.}
We warm-start our student model with the techniques developed by~\citet{shleifer2020pre}, who combine MLE training, word-level KL, and hidden state matching. Such a pre-distilling process is crucial to our \textsc{$f$-distill} method, because most variants (namely, RKL, JS, and TVD distillations) require sampling from a student, but a randomly initialized student model generates poor samples, making the distilling process less meaningful.

Notice that, for a fair comparison, all baseline models are built upon the same pre-distilling process. This further confirms that our \textsc{$f$-distill} is compatible with existing techniques and yields add-on performance gain (shown in Section~\ref{subsection:results}).

\section{Experiments}

\begin{table*}[!t]

\begin{subtable}[t]{\textwidth}
\centering
\resizebox{0.768\textwidth}{!}{
\begin{tabular}{|ll|cccccc|}
\hline
\multicolumn{2}{|l|}{\multirow{2}{*}{Model}} & \multicolumn{6}{c|}{DART}                      \\ \cline{3-8} 
\multicolumn{2}{|l|}{} &
  BLEU4$^\uparrow$ &
  METEOR$^\uparrow$ &
  TER$^\downarrow$ &
  BERTScore$^\uparrow$ &
  MoverScore$^\uparrow$ &
  BLEURT$^\uparrow$ \\ \hline
\multicolumn{2}{|l|}{Teacher}                & 48.56 & 39.28 & 45.45 & 83.04 & 68.17 & 40.56 \\ \hline
\multicolumn{1}{|l|}{\multirow{8}{*}{Student}} &
  Non-distill (MLE) &
  43.12 &
  35.71 &
  49.97 &
  79.76 &
  65.65 &
  29.10 \\
\multicolumn{1}{|l|}{}     & Pre-distill     & 45.60 & 36.99 & 47.10 & 81.39 & 66.75 & 34.08 \\ \cline{2-8} 
\multicolumn{1}{|l|}{}     & SeqKD           & 45.54 & 37.17 & 47.49 & 81.15 & 66.65 & 32.88 \\
\multicolumn{1}{|l|}{}     & ENGINE          & 44.40 & 36.51 & 50.63 & 80.18 & 66.20 & 30.94 \\ \cline{2-8} 
\multicolumn{1}{|l|}{}     & KL              & 46.24 & 37.45 & 46.89 & 81.60 & 67.07 & 35.31 \\
\multicolumn{1}{|l|}{}     & RKL             & 45.63 & 37.35 & 47.91 & 81.41 & 67.02 & 35.08 \\
\multicolumn{1}{|l|}{} &
  JS &
  \underline{46.85} &
  \underline{37.75} &
  \underline{46.50} &
  \underline{81.93} &
  \underline{67.30} &
  \underline{36.81} \\
\multicolumn{1}{|l|}{} &
  TVD &
  \textbf{46.95} &
  \textbf{37.88} &
  \textbf{46.35} &
  \textbf{82.08} &
  \textbf{67.36} &
  \textbf{37.17} \\ \hline
\end{tabular}
}
\label{tab:dart-xsum}
\end{subtable}

\vspace{.05cm}
\begin{subtable}[t]{\textwidth}
\centering
\resizebox{\textwidth}{!}{
\begin{tabular}{|ll|ccc|ccc|ScS|}
\hline
\multicolumn{2}{|l|}{\multirow{2}{*}{Model}} &
  \multicolumn{3}{c|}{XSum} &
  \multicolumn{3}{c|}{WMT16 EN-RO} &
  \multicolumn{3}{c|}{Commonsense Dialogue} \\ \cline{3-11} 
\multicolumn{2}{|l|}{} &
  \multicolumn{1}{c}{ROUGE-1$^\uparrow$} &
  \multicolumn{1}{c}{ROUGE-2$^\uparrow$} &
  \multicolumn{1}{c|}{ROUGE-L$^\uparrow$} &
  \multicolumn{1}{c}{BLEU4$^\uparrow$} &
  \multicolumn{1}{c}{chrF$^\uparrow$} &
  \multicolumn{1}{c|}{TER$^\downarrow$} &
  \multicolumn{1}{c}{BLEU1$^\uparrow$} &
  \multicolumn{1}{c}{BLEU2$^\uparrow$} &
  \multicolumn{1}{c|}{BERTScore$^\uparrow$} \\ \hline
\multicolumn{2}{|l|}{Teacher} &
  45.12 &
  22.26 &
  37.18 &
  25.82 &
  55.76 &
  60.57 &
  11.67 &
   5.03 &
  47.69 \\ \hline
\multicolumn{1}{|l|}{\multirow{8}{*}{Student}} &
  Non-distill (MLE) &
  30.00 &
  10.67 &
  24.40 &
  19.90 &
  49.79 &
  69.48 &
  10.23 &
   3.56 &
  45.15 \\
\multicolumn{1}{|l|}{} &
  Pre-distill &
  40.58 &
  17.79 &
  32.55 &
  20.68 &
  50.51 &
  68.38 &
   9.95 &
   3.63 &
  46.22 \\ \cline{2-11} 
\multicolumn{1}{|l|}{} &
  SeqKD &
  39.13 &
  17.53 &
  32.34 &
  21.20 &
  50.81 &
  67.66 &
  10.85 &
  4.17  &
  46.94 \\
\multicolumn{1}{|l|}{} &
  ENGINE &
  39.19 &
  16.18 &
  31.23 &
  17.65 &
  48.37 &
  84.02 &
  10.13 &
  4.26  &
  46.91 \\ \cline{2-11} 
\multicolumn{1}{|l|}{} &
  KL &
  41.28 &
  18.98 &
  33.71 &
  21.45 &
  51.12 &
  \textbf{66.74} &
  9.81 &
  3.52 &
  45.80 \\
\multicolumn{1}{|l|}{} &
  RKL &
  \underline{41.69} &
  19.02 &
  33.92 &
  20.46 &
  50.33 &
  70.78 &
  10.48 &
  4.01  &
  46.68 \\
\multicolumn{1}{|c|}{} &
  JS &
  41.65 &
  \underline{19.22} &
  \underline{34.03} &
  \textbf{21.91} &
  \textbf{51.5} &
  \underline{66.86} &
  \textbf{11.55} &
  \textbf{4.83} &
  \textbf{47.61} \\
\multicolumn{1}{|l|}{} &
  TVD &
  \textbf{41.76} &
  \textbf{19.30} &
  \textbf{34.10} &
  \underline{21.73} &
  \underline{51.13} &
  66.94 &
  \underline{11.39} &
  \underline{4.73} &
  \underline{47.30} \\ \hline
\end{tabular}
}
\label{tab:wmt-cd}
\end{subtable}
\caption{Main results on the DART, XSum, WMT16 EN-RO, and Commonsense Dialogue (CD) datasets.
The best student result is in \textbf{bold} and the second best is \underline{underlined}.
$^{\uparrow/\downarrow}$The higher/lower, the better.
}
\label{tab:exp-rslt}
\end{table*}

\subsection{Settings}

\textbf{Datasets and metrics.} We evaluated \textsc{$f$-distill} on a wide range of text generation tasks.

\underline{$\bullet$ {DART.}} The DART dataset~\citep{nan-etal-2021-dart} is a popular data-to-text generation benchmark, where samples consist of structured data records and their corresponding text descriptions.
We report common string-matching metrics, BLEU~\citep{papineni-etal-2002-bleu}, METEOR~\citep{banerjee-lavie-2005-meteor}, and TER~\citep{snover-etal-2006-study}, as well as popular learned metrics, BERTScore~\citep{zhang2019bertscore}, MoverScore~\citep{zhao-etal-2019-moverscore}, and BLEURT~\citep{sellam-etal-2020-bleurt}.

\underline{$\bullet$ {XSum.}} Extreme Summarization~\cite[XSum,][]{narayan-etal-2018-xsum} is a large-scale dataset consisting of BBC articles and their one-sentence summaries.
We report ROUGE scores, the most widely used metrics for summarization~\citep{lin-2004-rouge}.

\underline{$\bullet$ WMT16 EN-RO.} This dataset contains parallel texts for English and Romanian, and is one of the commonly used machine translation datasets~\citep{bojar-etal-2016-wmt16}. We extracted 100K samples from the original dataset, as the teacher performance is nearly saturated at this size. We report BLEU, chrF~\citep{popovic-2015-chrf}, and TER scores for the translation quality, following existing machine translation literature~\citep{sennrich-etal-2016-neural,barrault-etal-2019-findings}.

\underline{$\bullet$ Commonsense Dialogue.} The Commonsense Dialogue dataset~\citep{zhou-etal-2021-commonsense} consists of dialogue sessions that are grounded on social contexts. We evaluated the output quality by BLEU and BERTScore. We only report BLEU1 and BLEU2, as higher-order BLEU scores are known to be unreliable for dialogue evaluation~\citep{liu-etal-2016-evaluate}.

\textbf{Model architectures.}
We evaluated \textsc{$f$-distill} using state-of-the-art teacher models for different tasks. 
We followed the encoder--decoder architecture and used BART~\citep{lewis-etal-2020-bart} as the teacher for DART and XSum.
We used T5~\citep{t52020}, another encoder--decoder model, for WMT16 EN-RO, as it excels at machine translation.
For Commonsense Dialogue, we followed~\newcite{zhang-etal-2020-dialogpt} and used DialoGPT, a decoder-only model pretrained on massive dialogue data.

Our student models followed the teachers' architectures, but we reduced the number of layers. In our experiments, we generally set the total number of layers to be four; specifically, encoder--decoder models had three encoder layers and one decoder layer, following the suggestion of deep encoders and shallow decoders in \citet{kasai2020deep}. For XSum, we set both the encoder and decoder to be three layers to compensate for the larger dataset. Additional experimental details can be found in Appendix~\ref{apdx:exp-details}.

\subsection{Results and Analyses} \label{subsection:results}

\begin{table*}[!t]
\centering
\resizebox{0.60\textwidth}{!}{
\begin{tabular}{|l|ll|ll|ll|ll|}
\hline
Dataset &
  \multicolumn{2}{c|}{DART} &
  \multicolumn{2}{c|}{XSum} &
  \multicolumn{2}{c|}{MT$_\text{EN-RO}$} &
  \multicolumn{2}{c|}{CD} \\ \hline
TeacherDist &
  \multicolumn{2}{c|}{26.10} &
  \multicolumn{2}{c|}{36.28} &
  \multicolumn{2}{c|}{23.13} &
  \multicolumn{2}{c|}{81.19} \\\hline\hline
Risk &
  $R_\text{llh}$ &
  $R_\text{cvg}$ &
  $R_\text{llh}$ &
  $R_\text{cvg}$ &
  $R_\text{llh}$ &
  $R_\text{cvg}$ &
  $R_\text{llh}$ &
  $R_\text{cvg}$ \\ \hline
KL  & 0.56 & 0.49 & 1.89 & 1.68 & 1.23 & 0.82 & 0.43 & 0.26 \\
RKL & 0.58 & 0.59 & 1.88 & 1.83 & 1.20 & 1.60 & 0.29 & 0.35 \\
TVD & 0.53 & 0.52 & 1.86 & 1.77 & 1.21 & 1.78 & 0.27 & 0.35 \\
JS  & 0.51 & 0.48 & 1.88 & 1.75 & 1.13 & 1.34 & 0.30 & 0.33 \\ \hline
\end{tabular}
}\vspace*{-0.2cm}
\caption{The likelihood risk $R_\text{llh}$ and the coverage risk $R_\text{cvg}$ for different \textsc{$f$-distill} variants. A lower number indicates a higher likelihood or better coverage. We show the teacher diversity for each task by distinct bi-gram percentage~\citep{li-etal-2016-dist} among five teacher-sampled outputs given a test input, which indicates the severity of multi-modality of a task.
}\vspace*{-0.2cm}
\label{tab:mutual-nll}
\end{table*}

\textbf{Main results.}
 Table~\ref{tab:exp-rslt} presents the main results of our \fdistill\ along with a number of competing methods in the four experiments.
 
We first trained a neural network without distillation. The network was identical to our student model in terms of the neural architecture and hyperparameters, but we trained it directly by maximum likelihood estimation (MLE) based on ground-truth target sequences. 
As seen, the non-distilling model performs significantly worse than distilling methods, which agrees with existing literature and justifies the need for knowledge distillation~\citep{hinton2015distilling,tang2019distilling,jiao-etal-2020-tinybert}.

We pre-distilled our student model based on~\citet{shleifer2020pre}, a classic distilling approach that combines ground-truth training, word-level distillation, and intermediate-layer matching. Our \fdistill\ approach requires pre-distillation, because it provides a meaningful initialization of the student model, from which our \fdistill\ would generate samples during training. That being said, all our distilling methods were built on the same pre-distilling model, constituting a fair comparison. The results show that, although the pre-distilling approach outperforms ground-truth MLE training, it is generally worse than other distilling methods. This implies that our contribution is ``orthogonal'' to existing methods, and that our \fdistill\ provides an add-on performance improvement.

We further experimented with SeqKD~\citep{kim-rush-2016-sequence} and ENGINE~\citep{tu-etal-2020-engine}, two established distilling methods in the distribution-matching category (see Section~\ref{sec:intro}). They learn from hard sequences rather than probabilities, and thus are hard approximations of our KL and RKL distillations, respectively (Section~\ref{subsec:classic-kd}). 
As seen, our soft label-based methods consistently outperform SeqKD and ENGINE. 
This suggests that soft labels (i.e., probabilities) provide more informative supervision signals than hard sentences for sequence-level distillation, which is consistent with early literature on classification tasks~\citep{bucilua2006model,hinton2015distilling}. 

Among our \fdistill{} variants, we further observe that symmetric distilling losses (JS and TVD) are consistently better than asymmetric ones (KL and RKL) across all datasets except for WMT16 EN-RO, where KL achieves a slightly better TER performance.
A plausible reason is that the machine translation task is semantically grounded: given a source text, there are limited ways to translate, because the model output has to preserve the meaning of the input sentence. This is analogous to learning a uni-modal distribution, where mode averaging does not occur because there is only one mode. Despite this, JS and TVD perform better in all other scenarios, as their symmetric divergence can force the student to better learn from its teacher distribution. They rank first or second for all tasks in terms of most of the metrics in Table~\ref{tab:exp-rslt}, consistently and largely outperforming previous methods.

\textbf{Likelihood and coverage.}
We further analyze the mode averaging and collapsing behaviors of different distilling methods in Table~\ref{tab:mutual-nll}.
We propose to measure these aspects by a likelihood risk $R_\text{llh}$ and a coverage risk $R_\text{cvg}$.

The \textit{likelihood risk} is computed by
$R_\text{llh} = \frac{1}{|\mathcal D_\text{student}|}\sum\nolimits_{{\mathbf y}'\in \mathcal D_\text{student}} - \log p({\mathbf y}')$. 
Here,  $\mathcal D_\text{student}$ is the set of sentences generated from the student, where we sample a sentence for each input in the test set; $p({\mathbf y}')$ is the teacher's predicted probability of a student-sampled sentence ${\mathbf y}'$.
A large likelihood risk suggests that the student may have averaged the teacher's modes, causing it to generate atypical sentences from the teacher's point of view~(Figure~\ref{fig:modes}a).

On the contrary, the \textit{coverage risk} is computed by
$R_\text{cvg} = \frac{1}{|\mathcal D_\text{teacher}|}\sum\nolimits_{{\mathbf y}\in \mathcal D_\text{teacher}} - \log q_\theta({\mathbf y})
$, 
where we use the student $q_\theta$ to evaluate a teacher-sampled sentence $\mathbf y\in \mathcal D_\text{teacher}$. This measures whether the teacher's samples are typical from the student's point of view, i.e., how well a student covers the support of the teacher's distribution.
A large coverage risk means that the teacher's typical outputs are not captured by the student, which is an indicator of mode collapse (Figure~\ref{fig:modes}b).

In addition, we notice that mode averaging and collapsing are significantly affected by how ``multi-modal'' a task is.
We propose to measure this by the distinct bi-gram percentage~\citep{li-etal-2016-dist} of the teacher model (denoted by TeacherDist): for each test input, we sampled five outputs from the teacher and computed the percentage of distinct bi-grams, which is then averaged across the test set. As seen in Table~\ref{tab:mutual-nll}, the dialogue task exhibits the highest diversity, i.e., it is the most multi-modal, whereas machine translation is the least multi-modal.

Comparing KL and RKL, we find that KL distillation consistently achieves lower $ R_\text{cvg}$ risks (i.e., better coverage) than RKL across all datasets.
This confirms that KL distillation yields a smooth student distribution that covers the teacher's, whereas RKL distillation does not have the covering property due to its mode-collapsing nature.

We further observe that RKL achieves significantly higher likelihood (given by a lower $R_\text{llh}$) on the Commonsense Dialogue dataset.
This shows that the mode-collapsing phenomenon of RKL distillation allows the student to generate plausible responses for the one-to-many dialogue task~(Figure~\ref{fig:modes}b), whereas the mode-averaging KL distillation puts the student in some desolate area in the teacher's distribution (Figure~\ref{fig:modes}a).
On the other hand, RKL does not achieve lower likelihood risks in other tasks, since their one-to-many phenomenon is not as severe as dialogue generation~\citep{wei2019neural, bao-etal-2020-plato, wen2022equal}.

Referring back to Table~\ref{tab:exp-rslt}, we see that mode-averaging KL distillation is preferred over RKL for less multi-modal tasks, such as machine translation (which has a low TeacherDist score), whereas mode-collapsing RKL is preferred for highly multi-modal tasks, such as dialogue generation (which has a higher TeacherDist score).

Last, our symmetric distilling objectives (JS and TVD) generally have moderate likelihood and coverage risks between the two extremes.
This shows that they achieve a compromise between mode collapsing and averaging, allowing them to yield high performance in all tasks (Table~\ref{tab:exp-rslt}).

\textbf{Analysis of the student size.} We analyze our \textsc{$f$-distill} variants with different student sizes in comparison with the SeqKD model. Due to the limited time and resources, we chose the DART dataset as our testbed. We reduced the student model to different sizes by changing the number of encoder layers, as we had already used a single-layer decoder following the suggested architecture in~\citet{kasai2020deep}. Results are shown in Figure~\ref{fig:student-scaling}.

As seen, our \fdistill{} outperforms SeqKD across all model sizes.
The symmetric losses (JS and TVD) also consistently outperform the asymmetric ones (KL and RKL).
This is consistent with our main results and further validates the effectiveness and robustness of our \textsc{$f$-distill} framework.

\begin{figure}[!t]
\vspace*{-0.2cm}
    \centering
    \includegraphics[width=0.75\columnwidth]{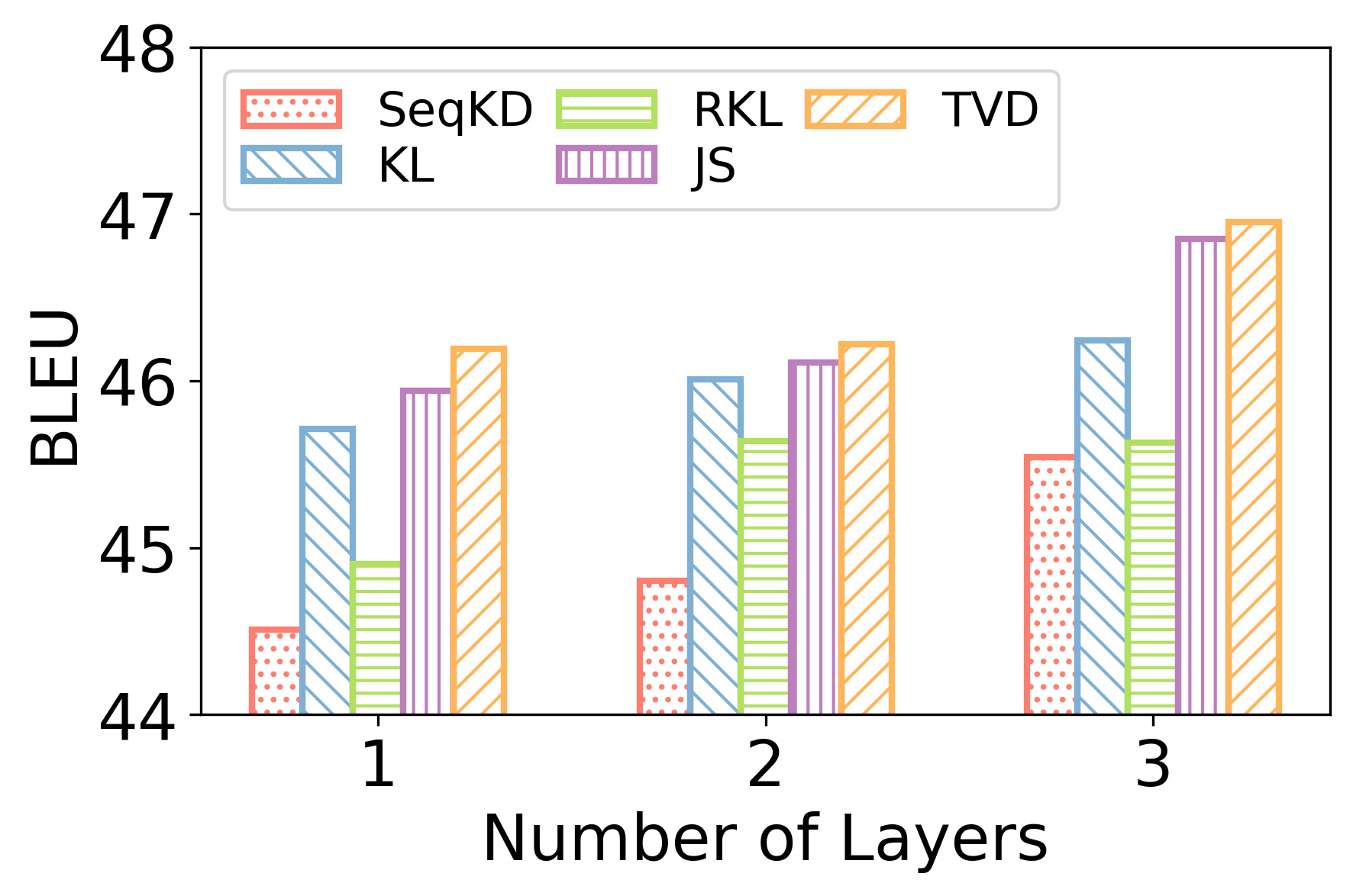}
    \vspace*{-0.2cm}\caption{Comparison of KD methods when the student model has different numbers of encoder layers. Results were obtained on the DART dataset.}
\label{fig:student-scaling}
\vspace*{-0.4cm}
\end{figure}
\begin{table}[!t]
\bigskip
\centering
\resizebox{\columnwidth}{!}{
\begin{tabular}{|l|ccc|}
\hline
Model         & BLEU4 & BERTScore & Speedup \\ \hline\hline
\multicolumn{4}{|c|}{JS distillation} \\\hline
Online            & 46.85 & 82.02  & 1.00x   \\
Offline (our method)  & 46.85 & 81.93  & 2.25x   \\ \hline\hline
\multicolumn{4}{|c|}{TVD distillation} \\\hline
Online           & 46.57 & 82.03  & 1.00x   \\
Offline (our method) & 46.95 & 82.08  & 2.31x   \\ \hline
\end{tabular}
}
\caption{Training efficiency on the DART dataset. \textbf{Online:} We re-sample sequences from the teacher model in every epoch. \textbf{Offline:} The teacher's samples are obtained beforehand and fixed during training. Note that we always re-sample from the student model because the student is constantly being updated.
}
\label{tab:efficiency}\vspace*{-0.2cm}
\end{table}
\textbf{Analysis of training efficiency.} Our \fdistill\ involves sampling sequences from the teacher. We propose an offline approach that obtains the teacher's samples before training.
We analyze the efficiency of offline sampling for JS and TVD distillations by comparing them with their online counterparts.
We ran this experiment on an NVidia RTX A6000 GPU and an Intel Xeon Gold 5317 CPU.\footnote{To obtain a rigorous time estimate, we ran efficiency analysis on an unshared, consumer-grade server, whereas other experiments were run on clusters (Appendix~\ref{apdx:exp-details}).}

As seen in Table~\ref{tab:efficiency}, the offline variant achieves comparable performance, while the training speed is more than doubled. This is expected, as the offline distilling methods do not require inference from the teacher model during training, which constitutes a significant portion of the training process.
This shows that our symmetric distilling methods can achieve high performance without the need for sampling from both the teacher and student.

\textbf{Human Evaluation.}
We further validated \fdistill\ by human evaluation, where models were rated by fluency, missing information, and hallucination between 1 to 5 on the DART dataset, following previous work~\citep{nan-etal-2021-dart,keymanesh-etal-2022-makes}.
We invited five human annotators to evaluate 50 test samples for four competing models: SeqKD, ENGINE, JS, and TVD.
For each test sample, the annotators were presented with shuffled model outputs, so they could not tell which output was generated by which model.
Results are shown in Table~\ref{tab:he}.

As seen, our \fdistill\ enables students to capture the input data records more faithfully while also retaining a high level of fluency. This is additionally supported by the $p$-values: comparing SeqKD and TVD, there is no statistically significant difference in terms of fluency ($p$-value=32.6\%); however, the improvements for missing information ($p$-value=1.28\%) and hallucination ($p$-value=0.669\%) are statistically significant.
Our human evaluation confirms the effectiveness of \fdistill.

\begin{table}[]
\centering
\resizebox{\columnwidth}{!}{
\begin{tabular}{|l|ccc|}
\hline
Model & Fluency$^\uparrow$ & MissingInfo$^\downarrow$ & Hallucination$^\downarrow$ \\ \hline
SeqKD  & \textbf{4.75} & 1.77        & 1.67 \\
ENGINE & 4.51 & 1.76        & 1.61 \\ \hline
JS     & \underline{4.72} & \underline{1.70}     & \underline{1.48} \\
TVD    & \underline{4.72} & \textbf{1.57}        & \textbf{1.45} \\ \hline
\end{tabular}
}
\caption{Human evaluation on the DART dataset. Comparing SeqKD and TVD, the one-sided Student's $t$-test gives $p$-values of  32.6\%, 1.28\%, and 0.669\% for fluency, missing information, and hallucination, respectively.}
\label{tab:he}
\end{table}

\textbf{Case Study.}
Appendix~\ref{apdx:case-study} shows example outputs for our \fdistill\ variants. Indeed, we observe KL distillation yields short and generic utterances that are believed to be an indicator of mode averaging~\cite{wei2019neural,bao-etal-2020-plato}. Our symmetric losses (JS and TVD) are able to generate more meaningful, fluent, and coherent sentences.

\section{Related Work}

Knowledge distillation (KD) is pioneered by \citet{bucilua2006model}, who use an ensemble model as the teacher to train a single-model student by minimizing the squared difference between their predicted logits.
\citet{hinton2015distilling} propose to directly learn from the output probabilities by minimizing their KL divergence.
\citet{sun2019pkd} propose patient knowledge distillation (PKD), which requires the student to learn from the teacher's intermediate layers.
\citet{jiao-etal-2020-tinybert} propose TinyBERT, extending knowledge distillation for Transformer models by additional treatments on the attention layers.
Other recent distilling methods include finding the optimal layer mapping between two models~\citep{li2020bertEMD, jiao2021improving} and learning from multiple teachers~\citep{multi-teacher-3,wu2021one,li-etal-2022-unsupervised-multiple}.

The success of KD has since sparked significant interest in its applications to text generation.
\citet{kim-rush-2016-sequence} investigate sequence-level knowledge distillation (SeqKD) for neural machine translation, where they use sampled, hard sequences to approximate the KL divergence.  \citet{tu-etal-2020-engine} train a student model by minimizing the energy function defined by a teacher model, which we show is an approximation to reverse KL distillation.
\citet{lin-etal-2020-autoregressive} propose imitation-based KD, where the teacher provides oracle probabilities on student-sampled partial sequences to address the exposure bias problem.
Further, KD has been extensively used to train non-autoregressive text generation models to reduce the complexity of the training data~\citep{gu2018non,shao-etal-2022-one,Huang_Zhou_Zaïane_Mou_Li_2022}.

It is noted that our \textsc{$f$-distill} requires meaningful student sampling and thus is built upon existing KD techniques~\citep{shleifer2020pre}, including word-level and intermediate-layer KD. Nevertheless, it shows that our approach achieves an add-on performance improvement, and that our contributions are orthogonal to previous work. 

Besides KD, common model compression techniques include parameter pruning and sparse modeling.
Parameter pruning first trains a dense network and then removes certain neural weights in hopes of not significantly affecting the model performance~\citep{classic-pruning,liu-etal-2018-efficient,fan-etal-2021-layer}.
Alternatively, one may apply sparse modeling techniques such as regularization during the training process to ensure zero-valued parameters~\citep{frankle2018lottery,louizos2018learning,Tang_Zhao_Wang_Luo_Xie_Zeng_2022}.
Our work does not follow these directions, as we consider the knowledge distilling setting.

Regarding the $f$-divergence function, it has many applications in the machine learning literature.
The standard cross-entropy training is equivalent to minimizing the KL divergence between the ground-truth label distribution (often one-hot) and model distribution~\cite{bishop2006pattern}.
Generative adversarial networks~\citep{goodfellow2014GAN} minimize the Jensen--Shannon divergence by simultaneously training a generator and a discriminator against each other.
\citet{zhao2020rkl} minimize  $\alpha$-divergence for adversarial learning, which generalizes KL and RKL,  and is a special case of $f$-divergence functions.
\citet{pmlr-v139-zhang21n} use total variation distance as a regularizer to encourage the model to predict more distinguishable probabilities.
Further, JSD is used in computer vision KD~\cite{Yin_2020_CVPR,NEURIPS2021_63dc7ed1}, but their tasks do not involve sequential data and the underlying techniques largely differ from our approach.
To the best of our knowledge, we are the first to systematically formulate sequence-level knowledge distillation as $f$-divergence minimization.

\section{Conclusion}

We propose \textsc{$f$-distill}, a family of sequence-level distilling methods beyond minimizing the KL divergence.
Under our framework, we propose and analyze four variants: KL, RKL, JS, and TVD distillations, where existing SeqKD and ENGINE are approximations of KL and RKL  variants; we further derive step-wise decomposition for our \textsc{$f$-distill}.
Results on four text generation tasks show \textsc{$f$-distill} consistently outperforms existing KD methods, and that our symmetric losses (JS and TVD) outperform asymmetric ones by avoiding extreme mode averaging and collapsing.

\section{Limitations}

Our \textsc{$f$-distill} variants are less efficient to train than SeqKD and ENGINE, as we require the teacher's soft probabilities instead of hard, sampled sequences.
However, our methods achieve a significant performance improvement, and more importantly, the additional training time does not affect inference when the model is deployed. This follows the spirit of knowledge distillation in general, i.e., to obtain a small and efficient model for deployment.

Another potential threat to validity is that we have not reported multi-run statistics. In our preliminary experiments, we ran our approach multiple times and found results were generally consistent. Due to our excessive experimentation (estimated at 2000 GPU hours), it is not possible to run each model multiple times. 
 We instead adopted a wide range of established automatic metrics, consistently showing the effectiveness of our approach. We further conducted in-depth analyses to better understand our proposed framework. 
We deem multi-run statistics not crucial to this paper, as this paper does not purely focus on empirical analysis. Rather, our main contributions lie in the novel machine learning framework,  \fdistill, and the theoretical connections between step-wise and sequence-level $f$-divergence functions.

Finally, a limitation of the formally published paper in the proceedings of ACL'23 is that we could not foresee the discussions during the conference, which are highlighted in the next section of this arXiv manuscript.

\section{Notes from ACL'23 Conference}
During the ACL conference, we received a number of questions and feedback, based on which we would like to make two clarifications.

1) Step-wise decomposition for JS distillation requires the conditional of the middle distribution $m(\mathrm Y_t | \mathbf Y_{1:{t-1}})$. Although it is tempting to have 
\begin{align} m(\mathrm Y_t | \mathbf Y_{1:{t-1}}) = \frac12 p(\mathrm Y_t | \mathbf Y_{1:{t-1}}) + \frac12 q_\theta(\mathrm Y_t | \mathbf Y_{1:{t-1}})\label{eq:wrong}
\end{align}
the correct formula should be 
\begin{align}m(\mathrm Y_t | \mathbf Y_{1:{t-1}}) = \frac{m(\mathbf Y_{1:t})}{ m(\mathbf Y_{1:t-1})}\label{eq:correct}
\end{align}
which is not the same as \eqref{eq:wrong}. We realize that we made the above mistake in our implementation, which nevertheless works empirically well and can be thought of as an approximation. It is also noted that calculating \eqref{eq:correct} requires storing the probabilities for each step and would be less efficient.

2) Certain \fdistill\ variants (RKL, JS, and TVD) involve sampling from the student.  Although we have derived step-wise decompositions in Eqns.~\eqref{eq:rkl}, \eqref{eq:JS-main}, and \eqref{eq:TVD}, directly applying backpropogation to them is unable to give the partial gradient with respect to the sampling parameters. In other words, a stop-gradient operation is performed for $q_\theta$ under the expectation~$\mathbb E[\cdot]$. 
A direct optimization requires reinforcement learning, which is known to be unstable for text generation~\citep{Yu_Zhang_Wang_Yu_2017} and is not considered in our approach.
On the other hand, our method optimizes the parameters greedily step by step, which nevertheless pushes the student distribution towards the teacher distribution: a loss of zero is achieved when the student and teacher are the same, no matter whether we propagate the gradient to the sampling distribution.
In this way, we are able to achieve meaningful results, while circumventing the difficulties of RL.

\section*{Acknowledgments}

We thank Wai Hong Ong for discussing the technical details of $m(\cdot)$ in JS distillation, and Lucas Torroba Hennigen for discussing the connection between \fdistill\ and reinforcement learning.

We also thank all reviewers and chairs for their valuable comments. The research is supported in part by the Natural Sciences and Engineering Research Council of Canada (NSERC) under Grant No. RGPIN2020-04465, the Amii Fellow Program, the Canada CIFAR AI Chair Program, a UAHJIC project, a donation from DeepMind, and the Digital Research Alliance of Canada (alliancecan.ca).

\bibliography{custom}

\begin{thebibliography}{63}
\expandafter\ifx\csname natexlab\endcsname\relax\def\natexlab#1{#1}\fi

\bibitem[{Ali and Silvey(1966)}]{classic-fdiv}
S.~M. Ali and S.~D. Silvey. 1966.
\newblock \href {http://www.jstor.org/stable/2984279} {A general class of
  coefficients of divergence of one distribution from another}.
\newblock \emph{Journal of the Royal Statistical Society}, 28(1):131--142.

\bibitem[{Banerjee and Lavie(2005)}]{banerjee-lavie-2005-meteor}
Satanjeev Banerjee and Alon Lavie. 2005.
\newblock \href {https://aclanthology.org/W05-0909} {{METEOR}: An automatic
  metric for {MT} evaluation with improved correlation with human judgments}.
\newblock In \emph{Proceedings of the {ACL} Workshop on Intrinsic and Extrinsic
  Evaluation Measures for Machine Translation and/or Summarization}, pages
  65--72.

\bibitem[{Bao et~al.(2020)Bao, He, Wang, Wu, and Wang}]{bao-etal-2020-plato}
Siqi Bao, Huang He, Fan Wang, Hua Wu, and Haifeng Wang. 2020.
\newblock \href {https://aclanthology.org/2020.acl-main.9} {{PLATO}:
  Pre-trained dialogue generation model with discrete latent variable}.
\newblock In \emph{Proceedings of the Annual Meeting of the Association for
  Computational Linguistics}, pages 85--96.

\bibitem[{Barrault et~al.(2019)Barrault, Bojar, Costa-juss{\`a}, Federmann,
  Fishel, Graham, Haddow, Huck, Koehn, Malmasi, Monz, M{\"u}ller, Pal, Post,
  and Zampieri}]{barrault-etal-2019-findings}
Lo{\"\i}c Barrault, Ond{\v{r}}ej Bojar, Marta~R. Costa-juss{\`a}, Christian
  Federmann, Mark Fishel, Yvette Graham, Barry Haddow, Matthias Huck, Philipp
  Koehn, Shervin Malmasi, Christof Monz, Mathias M{\"u}ller, Santanu Pal, Matt
  Post, and Marcos Zampieri. 2019.
\newblock \href {https://aclanthology.org/W19-5301} {Findings of the 2019
  {C}onference on {M}achine {T}ranslation ({WMT}19)}.
\newblock In \emph{Proceedings of the Conference on Machine Translation}, pages
  1--61.

\bibitem[{Bishop(2006)}]{bishop2006pattern}
Christopher~M. Bishop. 2006.
\newblock \href {https://link.springer.com/book/9780387310732} {\emph{Pattern
  Recognition and Machine Learning}}.
\newblock Springer.

\bibitem[{Bojar et~al.(2016)Bojar, Chatterjee, Federmann, Graham, Haddow, Huck,
  Jimeno~Yepes, Koehn, Logacheva, Monz, Negri, N{\'e}v{\'e}ol, Neves, Popel,
  Post, Rubino, Scarton, Specia, Turchi, Verspoor, and
  Zampieri}]{bojar-etal-2016-wmt16}
Ond{\v{r}}ej Bojar, Rajen Chatterjee, Christian Federmann, Yvette Graham, Barry
  Haddow, Matthias Huck, Antonio Jimeno~Yepes, Philipp Koehn, Varvara
  Logacheva, Christof Monz, Matteo Negri, Aur{\'e}lie N{\'e}v{\'e}ol, Mariana
  Neves, Martin Popel, Matt Post, Raphael Rubino, Carolina Scarton, Lucia
  Specia, Marco Turchi, Karin Verspoor, and Marcos Zampieri. 2016.
\newblock \href {https://aclanthology.org/W16-2301} {Findings of the 2016
  {C}onference on {M}achine {T}ranslation}.
\newblock In \emph{Proceedings of the Conference on Machine Translation}, pages
  131--198.

\bibitem[{Buciluǎ et~al.(2006)Buciluǎ, Caruana, and
  Niculescu-Mizil}]{bucilua2006model}
Cristian Buciluǎ, Rich Caruana, and Alexandru Niculescu-Mizil. 2006.
\newblock \href {https://doi.org/10.1145/1150402.1150464} {Model compression}.
\newblock In \emph{Proceedings of the ACM SIGKDD International Conference on
  Knowledge Discovery and Data Mining}, page 535–541.

\bibitem[{Fan et~al.(2021)Fan, Li, Zhang, Ao, Wu, Meng, and
  Sun}]{fan-etal-2021-layer}
Chun Fan, Jiwei Li, Tianwei Zhang, Xiang Ao, Fei Wu, Yuxian Meng, and Xiaofei
  Sun. 2021.
\newblock \href {https://aclanthology.org/2021.emnlp-main.246} {Layer-wise
  model pruning based on mutual information}.
\newblock In \emph{Proceedings of the Conference on Empirical Methods in
  Natural Language Processing}, pages 3079--3090.

\bibitem[{Fang et~al.(2021)Fang, Bao, Song, Wang, Xie, Shen, and
  Song}]{NEURIPS2021_63dc7ed1}
Gongfan Fang, Yifan Bao, Jie Song, Xinchao Wang, Donglin Xie, Chengchao Shen,
  and Mingli Song. 2021.
\newblock \href
  {https://proceedings.neurips.cc/paper_files/paper/2021/file/63dc7ed1010d3c3b8269faf0ba7491d4-Paper.pdf}
  {Mosaicking to distill: Knowledge distillation from out-of-domain data}.
\newblock In \emph{Advances in Neural Information Processing Systems}, pages
  11920--11932.

\bibitem[{Frankle and Carbin(2018)}]{frankle2018lottery}
Jonathan Frankle and Michael Carbin. 2018.
\newblock \href {https://openreview.net/forum?id=rJl-b3RcF7} {The lottery
  ticket hypothesis: Finding sparse, trainable neural networks}.
\newblock In \emph{International Conference on Learning Representations}.

\bibitem[{Goodfellow et~al.(2014)Goodfellow, Pouget-Abadie, Mirza, Xu,
  Warde-Farley, Ozair, Courville, and Bengio}]{goodfellow2014GAN}
Ian Goodfellow, Jean Pouget-Abadie, Mehdi Mirza, Bing Xu, David Warde-Farley,
  Sherjil Ozair, Aaron Courville, and Yoshua Bengio. 2014.
\newblock \href
  {https://proceedings.neurips.cc/paper/2014/file/5ca3e9b122f61f8f06494c97b1afccf3-Paper.pdf}
  {Generative adversarial nets}.
\newblock In \emph{Advances in Neural Information Processing Systems}.

\bibitem[{Gu et~al.(2018)Gu, Bradbury, Xiong, Li, and Socher}]{gu2018non}
Jiatao Gu, James Bradbury, Caiming Xiong, Victor~OK Li, and Richard Socher.
  2018.
\newblock \href {https://openreview.net/forum?id=B1l8BtlCb&utm_campaign}
  {Non-autoregressive neural machine translation}.
\newblock In \emph{International Conference on Learning Representations}.

\bibitem[{Hinton et~al.(2015)Hinton, Vinyals, Dean
  et~al.}]{hinton2015distilling}
Geoffrey Hinton, Oriol Vinyals, Jeff Dean, et~al. 2015.
\newblock \href {https://arxiv.org/abs/1503.02531} {Distilling the knowledge in
  a neural network}.
\newblock \emph{arXiv preprint arXiv:1503.02531}.

\bibitem[{Huang et~al.(2022)Huang, Zhou, Zaïane, Mou, and
  Li}]{Huang_Zhou_Zaïane_Mou_Li_2022}
Chenyang Huang, Hao Zhou, Osmar~R. Zaïane, Lili Mou, and Lei Li. 2022.
\newblock \href {https://doi.org/10.1609/aaai.v36i10.21323} {Non-autoregressive
  translation with layer-wise prediction and deep supervision}.
\newblock In \emph{Proceedings of the AAAI Conference on Artificial
  Intelligence}, pages 10776--10784.

\bibitem[{Jiao et~al.(2021)Jiao, Chang, Yin, Shang, Jiang, Chen, Li, Wang, and
  Liu}]{jiao2021improving}
Xiaoqi Jiao, Huating Chang, Yichun Yin, Lifeng Shang, Xin Jiang, Xiao Chen,
  Linlin Li, Fang Wang, and Qun Liu. 2021.
\newblock \href
  {https://www.sciencedirect.com/science/article/pii/S0925231221010948}
  {Improving task-agnostic {BERT} distillation with layer mapping search}.
\newblock \emph{Neurocomputing}, 461:194--203.

\bibitem[{Jiao et~al.(2020)Jiao, Yin, Shang, Jiang, Chen, Li, Wang, and
  Liu}]{jiao-etal-2020-tinybert}
Xiaoqi Jiao, Yichun Yin, Lifeng Shang, Xin Jiang, Xiao Chen, Linlin Li, Fang
  Wang, and Qun Liu. 2020.
\newblock \href {https://aclanthology.org/2020.findings-emnlp.372}
  {{T}iny{BERT}: Distilling {BERT} for natural language understanding}.
\newblock In \emph{Findings of the Association for Computational Linguistics:
  EMNLP}, pages 4163--4174.

\bibitem[{Kasai et~al.(2020)Kasai, Pappas, Peng, Cross, and
  Smith}]{kasai2020deep}
Jungo Kasai, Nikolaos Pappas, Hao Peng, James Cross, and Noah Smith. 2020.
\newblock \href {https://openreview.net/forum?id=KpfasTaLUpq} {Deep encoder,
  shallow decoder: Reevaluating non-autoregressive machine translation}.
\newblock In \emph{International Conference on Learning Representations}.

\bibitem[{Keymanesh et~al.(2022)Keymanesh, Benton, and
  Dredze}]{keymanesh-etal-2022-makes}
Moniba Keymanesh, Adrian Benton, and Mark Dredze. 2022.
\newblock \href {https://aclanthology.org/2022.gem-1.50} {What makes
  data-to-text generation hard for pretrained language models?}
\newblock In \emph{Proceedings of the Workshop on Natural Language Generation,
  Evaluation, and Metrics}, pages 539--554.

\bibitem[{Kim and Rush(2016)}]{kim-rush-2016-sequence}
Yoon Kim and Alexander~M. Rush. 2016.
\newblock \href {https://www.aclweb.org/anthology/D16-1139} {Sequence-level
  knowledge distillation}.
\newblock In \emph{Proceedings of the Conference on Empirical Methods in
  Natural Language Processing}, pages 1317--1327.

\bibitem[{Kingma and Ba(2015)}]{adam-optimizer}
Diederik~P. Kingma and Jimmy Ba. 2015.
\newblock \href {http://arxiv.org/abs/1412.6980} {Adam: {A} method for
  stochastic optimization}.
\newblock In \emph{International Conference on Learning Representations}.

\bibitem[{Lebret et~al.(2016)Lebret, Grangier, and
  Auli}]{lebret-etal-2016-neural}
R{\'e}mi Lebret, David Grangier, and Michael Auli. 2016.
\newblock \href {https://aclanthology.org/D16-1128} {Neural text generation
  from structured data with application to the biography domain}.
\newblock In \emph{Proceedings of the Conference on Empirical Methods in
  Natural Language Processing}, pages 1203--1213.

\bibitem[{LeCun et~al.(1989)LeCun, Denker, and Solla}]{classic-pruning}
Yann LeCun, John Denker, and Sara Solla. 1989.
\newblock \href
  {https://papers.nips.cc/paper/1989/hash/6c9882bbac1c7093bd25041881277658-Abstract.html}
  {Optimal brain damage}.
\newblock In \emph{Advances in Neural Information Processing Systems}, pages
  598--605.

\bibitem[{Lewis et~al.(2020)Lewis, Liu, Goyal, Ghazvininejad, Mohamed, Levy,
  Stoyanov, and Zettlemoyer}]{lewis-etal-2020-bart}
Mike Lewis, Yinhan Liu, Naman Goyal, Marjan Ghazvininejad, Abdelrahman Mohamed,
  Omer Levy, Veselin Stoyanov, and Luke Zettlemoyer. 2020.
\newblock \href {https://aclanthology.org/2020.acl-main.703} {{BART}: Denoising
  sequence-to-sequence pre-training for natural language generation,
  translation, and comprehension}.
\newblock In \emph{Proceedings of the Annual Meeting of the Association for
  Computational Linguistics}, pages 7871--7880.

\bibitem[{Li et~al.(2020)Li, Liu, Zhao, Xu, Yang, and Jin}]{li2020bertEMD}
Jianquan Li, Xiaokang Liu, Honghong Zhao, Ruifeng Xu, Min Yang, and Yaohong
  Jin. 2020.
\newblock \href {https://aclanthology.org/2020.emnlp-main.242} {{BERT}-{EMD}:
  Many-to-many layer mapping for {BERT} compression with earth mover{'}s
  distance}.
\newblock In \emph{Proceedings of the Conference on Empirical Methods in
  Natural Language Processing}, pages 3009--3018.

\bibitem[{Li et~al.(2016{\natexlab{a}})Li, Galley, Brockett, Gao, and
  Dolan}]{li-etal-2016-dist}
Jiwei Li, Michel Galley, Chris Brockett, Jianfeng Gao, and Bill Dolan.
  2016{\natexlab{a}}.
\newblock \href {https://aclanthology.org/N16-1014/} {A diversity-promoting
  objective function for neural conversation models}.
\newblock In \emph{Proceedings of the Conference of the North {A}merican
  Chapter of the Association for Computational Linguistics: Human Language
  Technologies}, pages 110--119.

\bibitem[{Li et~al.(2016{\natexlab{b}})Li, Monroe, Ritter, Jurafsky, Galley,
  and Gao}]{dialogue-reinforcement}
Jiwei Li, Will Monroe, Alan Ritter, Dan Jurafsky, Michel Galley, and Jianfeng
  Gao. 2016{\natexlab{b}}.
\newblock \href {https://aclanthology.org/D16-1127} {Deep reinforcement
  learning for dialogue generation}.
\newblock In \emph{Proceedings of the Conference on Empirical Methods in
  Natural Language Processing}, pages 1192--1202.

\bibitem[{Li and Liang(2021)}]{li-liang-2021-prefix}
Xiang~Lisa Li and Percy Liang. 2021.
\newblock \href {https://aclanthology.org/2021.acl-long.353} {Prefix-tuning:
  Optimizing continuous prompts for generation}.
\newblock In \emph{Proceedings of the Annual Meeting of the Association for
  Computational Linguistics and the International Joint Conference on Natural
  Language Processing}, pages 4582--4597.

\bibitem[{Li et~al.(2022)Li, Hu, Guo, Chen, Qin, and
  Zhang}]{li-etal-2022-unsupervised-multiple}
Zhuoran Li, Chunming Hu, Xiaohui Guo, Junfan Chen, Wenyi Qin, and Richong
  Zhang. 2022.
\newblock \href {https://aclanthology.org/2022.acl-long.14} {An unsupervised
  multiple-task and multiple-teacher model for cross-lingual named entity
  recognition}.
\newblock In \emph{Proceedings of the Annual Meeting of the Association for
  Computational Linguistics}, pages 170--179.

\bibitem[{Lin et~al.(2020)Lin, Wohlwend, Chen, and
  Lei}]{lin-etal-2020-autoregressive}
Alexander Lin, Jeremy Wohlwend, Howard Chen, and Tao Lei. 2020.
\newblock \href {https://aclanthology.org/2020.emnlp-main.494} {Autoregressive
  knowledge distillation through imitation learning}.
\newblock In \emph{Proceedings of the Conference on Empirical Methods in
  Natural Language Processing}, pages 6121--6133.

\bibitem[{Lin(2004)}]{lin-2004-rouge}
Chin-Yew Lin. 2004.
\newblock \href {https://aclanthology.org/W04-1013} {{ROUGE}: A package for
  automatic evaluation of summaries}.
\newblock In \emph{Text Summarization Branches Out}, pages 74--81.

\bibitem[{Liu et~al.(2016)Liu, Lowe, Serban, Noseworthy, Charlin, and
  Pineau}]{liu-etal-2016-evaluate}
Chia-Wei Liu, Ryan Lowe, Iulian Serban, Mike Noseworthy, Laurent Charlin, and
  Joelle Pineau. 2016.
\newblock \href {https://aclanthology.org/D16-1230} {How {NOT} to evaluate your
  dialogue system: An empirical study of unsupervised evaluation metrics for
  dialogue response generation}.
\newblock In \emph{Proceedings of the Conference on Empirical Methods in
  Natural Language Processing}, pages 2122--2132.

\bibitem[{Liu et~al.(2018)Liu, Ren, Shang, Gu, Peng, and
  Han}]{liu-etal-2018-efficient}
Liyuan Liu, Xiang Ren, Jingbo Shang, Xiaotao Gu, Jian Peng, and Jiawei Han.
  2018.
\newblock \href {https://aclanthology.org/D18-1153} {Efficient contextualized
  representation: Language model pruning for sequence labeling}.
\newblock In \emph{Proceedings of the Conference on Empirical Methods in
  Natural Language Processing}, pages 1215--1225.

\bibitem[{Louizos et~al.(2018)Louizos, Welling, and
  Kingma}]{louizos2018learning}
Christos Louizos, Max Welling, and Diederik~P Kingma. 2018.
\newblock \href {https://openreview.net/forum?id=H1Y8hhg0b} {Learning sparse
  neural networks through ${L}_0$ regularization}.
\newblock In \emph{International Conference on Learning Representations}.

\bibitem[{Nan et~al.(2021)Nan, Radev, Zhang, Rau, Sivaprasad, Hsieh, Tang,
  Vyas, Verma, Krishna, Liu, Irwanto, Pan, Rahman, Zaidi, Mutuma, Tarabar,
  Gupta, Yu, Tan, Lin, Xiong, Socher, and Rajani}]{nan-etal-2021-dart}
Linyong Nan, Dragomir Radev, Rui Zhang, Amrit Rau, Abhinand Sivaprasad,
  Chiachun Hsieh, Xiangru Tang, Aadit Vyas, Neha Verma, Pranav Krishna,
  Yangxiaokang Liu, Nadia Irwanto, Jessica Pan, Faiaz Rahman, Ahmad Zaidi,
  Mutethia Mutuma, Yasin Tarabar, Ankit Gupta, Tao Yu, Yi~Chern Tan,
  Xi~Victoria Lin, Caiming Xiong, Richard Socher, and Nazneen~Fatema Rajani.
  2021.
\newblock \href {https://aclanthology.org/2021.naacl-main.37} {{DART}:
  Open-domain structured data record to text generation}.
\newblock In \emph{Proceedings of the Conference of the North American Chapter
  of the Association for Computational Linguistics: Human Language
  Technologies}, pages 432--447.

\bibitem[{Narayan et~al.(2018)Narayan, Cohen, and
  Lapata}]{narayan-etal-2018-xsum}
Shashi Narayan, Shay~B. Cohen, and Mirella Lapata. 2018.
\newblock \href {https://aclanthology.org/D18-1206} {Don{'}t give me the
  details, just the summary! {T}opic-aware convolutional neural networks for
  extreme summarization}.
\newblock In \emph{Proceedings of the Conference on Empirical Methods in
  Natural Language Processing}, pages 1797--1807.

\bibitem[{Papineni et~al.(2002)Papineni, Roukos, Ward, and
  Zhu}]{papineni-etal-2002-bleu}
Kishore Papineni, Salim Roukos, Todd Ward, and Wei-Jing Zhu. 2002.
\newblock \href {https://aclanthology.org/P02-1040} {{BLEU}: A method for
  automatic evaluation of machine translation}.
\newblock In \emph{Proceedings of the Annual Meeting of the Association for
  Computational Linguistics}, pages 311--318.

\bibitem[{Paulus et~al.(2018)Paulus, Xiong, and Socher}]{paulus2018deep}
Romain Paulus, Caiming Xiong, and Richard Socher. 2018.
\newblock \href {https://openreview.net/forum?id=HkAClQgA-} {A deep reinforced
  model for abstractive summarization}.
\newblock In \emph{International Conference on Learning Representations}.

\bibitem[{Popovi{\'c}(2015)}]{popovic-2015-chrf}
Maja Popovi{\'c}. 2015.
\newblock \href {https://aclanthology.org/W15-3049} {chr{F}: character n-gram
  {F}-score for automatic {MT} evaluation}.
\newblock In \emph{Proceedings of the Tenth Workshop on Statistical Machine
  Translation}, pages 392--395.

\bibitem[{Raffel et~al.(2020)Raffel, Shazeer, Roberts, Lee, Narang, Matena,
  Zhou, Li, Liu et~al.}]{t52020}
Colin Raffel, Noam Shazeer, Adam Roberts, Katherine Lee, Sharan Narang, Michael
  Matena, Yanqi Zhou, Wei Li, Peter~J Liu, et~al. 2020.
\newblock \href {https://jmlr.org/papers/v21/20-074.html} {Exploring the limits
  of transfer learning with a unified text-to-text {T}ransformer.}
\newblock \emph{Journal of Machine Learning Research}, 21(140):1--67.

\bibitem[{Sason and Verdú(2016)}]{ieee-fdiv}
Igal Sason and Sergio Verdú. 2016.
\newblock \href {https://ieeexplore.ieee.org/abstract/document/7552457}
  {$f$-divergence inequalities}.
\newblock \emph{IEEE Transactions on Information Theory}, 62(11):5973--6006.

\bibitem[{Sellam et~al.(2020)Sellam, Das, and Parikh}]{sellam-etal-2020-bleurt}
Thibault Sellam, Dipanjan Das, and Ankur Parikh. 2020.
\newblock \href {https://aclanthology.org/2020.acl-main.704} {{BLEURT}:
  Learning robust metrics for text generation}.
\newblock In \emph{Proceedings of the Annual Meeting of the Association for
  Computational Linguistics}, pages 7881--7892.

\bibitem[{Sennrich et~al.(2016)Sennrich, Haddow, and
  Birch}]{sennrich-etal-2016-neural}
Rico Sennrich, Barry Haddow, and Alexandra Birch. 2016.
\newblock \href {https://doi.org/10.18653/v1/P16-1162} {Neural machine
  translation of rare words with subword units}.
\newblock In \emph{Proceedings of the Annual Meeting of the Association for
  Computational Linguistics}, pages 1715--1725.

\bibitem[{Shao et~al.(2022)Shao, Wu, and Feng}]{shao-etal-2022-one}
Chenze Shao, Xuanfu Wu, and Yang Feng. 2022.
\newblock \href {https://aclanthology.org/2022.naacl-main.277} {One reference
  is not enough: Diverse distillation with reference selection for
  non-autoregressive translation}.
\newblock In \emph{Proceedings of the Conference of the North American Chapter
  of the Association for Computational Linguistics: Human Language
  Technologies}, pages 3779--3791.

\bibitem[{Shazeer and Stern(2018)}]{pmlr-v80-adafactor}
Noam Shazeer and Mitchell Stern. 2018.
\newblock \href {https://proceedings.mlr.press/v80/shazeer18a.html} {Adafactor:
  Adaptive learning rates with sublinear memory cost}.
\newblock In \emph{Proceedings of the International Conference on Machine
  Learning}, pages 4596--4604.

\bibitem[{Shleifer and Rush(2020)}]{shleifer2020pre}
Sam Shleifer and Alexander~M Rush. 2020.
\newblock \href {https://arxiv.org/abs/2010.13002} {Pre-trained summarization
  distillation}.
\newblock \emph{arXiv preprint arXiv:2010.13002}.

\bibitem[{Snover et~al.(2006)Snover, Dorr, Schwartz, Micciulla, and
  Makhoul}]{snover-etal-2006-study}
Matthew Snover, Bonnie Dorr, Rich Schwartz, Linnea Micciulla, and John Makhoul.
  2006.
\newblock \href {https://aclanthology.org/2006.amta-papers.25} {A study of
  translation edit rate with targeted human annotation}.
\newblock In \emph{Proceedings of the Conference of the Association for Machine
  Translation in the Americas: Technical Papers}, pages 223--231.

\bibitem[{Sun et~al.(2019)Sun, Cheng, Gan, and Liu}]{sun2019pkd}
Siqi Sun, Yu~Cheng, Zhe Gan, and Jingjing Liu. 2019.
\newblock \href {https://aclanthology.org/D19-1441} {Patient knowledge
  distillation for {BERT} model compression}.
\newblock In \emph{Proceedings of the Conference on Empirical Methods in
  Natural Language Processing and the International Joint Conference on Natural
  Language Processing}, pages 4323--4332.

\bibitem[{Tang et~al.(2022)Tang, Zhao, Wang, Luo, Xie, and
  Zeng}]{Tang_Zhao_Wang_Luo_Xie_Zeng_2022}
Chuanxin Tang, Yucheng Zhao, Guangting Wang, Chong Luo, Wenxuan Xie, and Wenjun
  Zeng. 2022.
\newblock \href {https://ojs.aaai.org/index.php/AAAI/article/view/20133}
  {Sparse {MLP} for image recognition: Is self-attention really necessary?}
\newblock In \emph{Proceedings of the AAAI Conference on Artificial
  Intelligence}, pages 2344--2351.

\bibitem[{Tang et~al.(2019)Tang, Lu, Liu, Mou, Vechtomova, and
  Lin}]{tang2019distilling}
Raphael Tang, Yao Lu, Linqing Liu, Lili Mou, Olga Vechtomova, and Jimmy Lin.
  2019.
\newblock \href {https://arxiv.org/abs/1903.12136} {Distilling task-specific
  knowledge from {BERT} into simple neural networks}.
\newblock \emph{arXiv preprint arXiv:1903.12136}.

\bibitem[{Tu et~al.(2020)Tu, Pang, Wiseman, and Gimpel}]{tu-etal-2020-engine}
Lifu Tu, Richard~Yuanzhe Pang, Sam Wiseman, and Kevin Gimpel. 2020.
\newblock \href {https://aclanthology.org/2020.acl-main.251} {{ENGINE}:
  Energy-based inference networks for non-autoregressive machine translation}.
\newblock In \emph{Proceedings of the Annual Meeting of the Association for
  Computational Linguistics}, pages 2819--2826.

\bibitem[{Wei et~al.(2019)Wei, Lu, Mou, Zhou, Poupart, Li, and
  Jin}]{wei2019neural}
Bolin Wei, Shuai Lu, Lili Mou, Hao Zhou, Pascal Poupart, Ge~Li, and Zhi Jin.
  2019.
\newblock \href {https://ieeexplore.ieee.org/document/8682634} {Why do neural
  dialog systems generate short and meaningless replies? {A} comparison between
  dialog and translation}.
\newblock In \emph{Proceedings of the International Conference on Acoustics,
  Speech and Signal Processing}, pages 7290--7294.

\bibitem[{Wen et~al.(2023)Wen, Hao, Cao, and Mou}]{wen2022equal}
Yuqiao Wen, Yongchang Hao, Yanshuai Cao, and Lili Mou. 2023.
\newblock \href {https://openreview.net/forum?id=k5PEHHY4spM} {An equal-size
  hard {EM} algorithm for diverse dialogue generation}.
\newblock In \emph{International Conference on Learning Representations}.

\bibitem[{Wu et~al.(2021)Wu, Wu, and Huang}]{wu2021one}
Chuhan Wu, Fangzhao Wu, and Yongfeng Huang. 2021.
\newblock \href {https://aclanthology.org/2021.findings-acl.387} {One teacher
  is enough? {P}re-trained language model distillation from multiple teachers}.
\newblock In \emph{Findings of the Association for Computational Linguistics:
  ACL-IJCNLP}, pages 4408--4413.

\bibitem[{Yang et~al.(2020)Yang, Shou, Gong, Lin, and Jiang}]{multi-teacher-3}
Ze~Yang, Linjun Shou, Ming Gong, Wutao Lin, and Daxin Jiang. 2020.
\newblock \href {https://doi.org/10.1145/3336191.3371792} {Model compression
  with two-stage multi-teacher knowledge distillation for web question
  answering system}.
\newblock In \emph{Proceedings of the International Conference on Web Search
  and Data Mining}, page 690–698.

\bibitem[{Yin et~al.(2020)Yin, Molchanov, Alvarez, Li, Mallya, Hoiem, Jha, and
  Kautz}]{Yin_2020_CVPR}
Hongxu Yin, Pavlo Molchanov, Jose~M. Alvarez, Zhizhong Li, Arun Mallya, Derek
  Hoiem, Niraj~K. Jha, and Jan Kautz. 2020.
\newblock \href
  {https://openaccess.thecvf.com/content_CVPR_2020/html/Yin_Dreaming_to_Distill_Data-Free_Knowledge_Transfer_via_DeepInversion_CVPR_2020_paper.html}
  {Dreaming to distill: Data-free knowledge transfer via {D}eep{I}nversion}.
\newblock In \emph{Proceedings of the IEEE/CVF Conference on Computer Vision
  and Pattern Recognition}, pages 8715--8724.

\bibitem[{Yu et~al.(2017)Yu, Zhang, Wang, and Yu}]{Yu_Zhang_Wang_Yu_2017}
Lantao Yu, Weinan Zhang, Jun Wang, and Yong Yu. 2017.
\newblock \href {https://ojs.aaai.org/index.php/AAAI/article/view/10804}
  {Seq{GAN}: Sequence generative adversarial nets with policy gradient}.
\newblock In \emph{Proceedings of the AAAI Conference on Artificial
  Intelligence}, pages 2852--2858.

\bibitem[{Zhang et~al.(2020{\natexlab{a}})Zhang, Zhao, Saleh, and
  Liu}]{pmlr-v119-zhang20ae}
Jingqing Zhang, Yao Zhao, Mohammad Saleh, and Peter Liu. 2020{\natexlab{a}}.
\newblock \href {https://proceedings.mlr.press/v119/zhang20ae.html} {{PEGASUS}:
  Pre-training with extracted gap-sentences for abstractive summarization}.
\newblock In \emph{Proceedings of the International Conference on Machine
  Learning}, pages 11328--11339.

\bibitem[{Zhang et~al.(2019)Zhang, Kishore, Wu, Weinberger, and
  Artzi}]{zhang2019bertscore}
Tianyi Zhang, Varsha Kishore, Felix Wu, Kilian~Q Weinberger, and Yoav Artzi.
  2019.
\newblock \href {https://openreview.net/forum?id=SkeHuCVFDr} {{BERTS}core:
  Evaluating text generation with {BERT}}.
\newblock In \emph{International Conference on Learning Representations}.

\bibitem[{Zhang et~al.(2021)Zhang, Niu, and Sugiyama}]{pmlr-v139-zhang21n}
Yivan Zhang, Gang Niu, and Masashi Sugiyama. 2021.
\newblock \href {https://proceedings.mlr.press/v139/zhang21n.html} {Learning
  noise transition matrix from only noisy labels via total variation
  regularization}.
\newblock In \emph{Proceedings of the International Conference on Machine
  Learning}, pages 12501--12512.

\bibitem[{Zhang et~al.(2020{\natexlab{b}})Zhang, Sun, Galley, Chen, Brockett,
  Gao, Gao, Liu, and Dolan}]{zhang-etal-2020-dialogpt}
Yizhe Zhang, Siqi Sun, Michel Galley, Yen-Chun Chen, Chris Brockett, Xiang Gao,
  Jianfeng Gao, Jingjing Liu, and Bill Dolan. 2020{\natexlab{b}}.
\newblock \href {https://aclanthology.org/2020.acl-demos.30} {{DIALOGPT} :
  Large-scale generative pre-training for conversational response generation}.
\newblock In \emph{Proceedings of the Annual Meeting of the Association for
  Computational Linguistics: System Demonstrations}, pages 270--278.

\bibitem[{Zhao et~al.(2020)Zhao, Cong, Dai, and Carin}]{zhao2020rkl}
Miaoyun Zhao, Yulai Cong, Shuyang Dai, and Lawrence Carin. 2020.
\newblock \href {https://ojs.aaai.org/index.php/AAAI/article/view/6172}
  {Bridging maximum likelihood and adversarial learning via
  $\alpha$-divergence}.
\newblock In \emph{Proceedings of the AAAI Conference on Artificial
  Intelligence}, pages 6901--6908.

\bibitem[{Zhao et~al.(2019)Zhao, Peyrard, Liu, Gao, Meyer, and
  Eger}]{zhao-etal-2019-moverscore}
Wei Zhao, Maxime Peyrard, Fei Liu, Yang Gao, Christian~M. Meyer, and Steffen
  Eger. 2019.
\newblock \href {https://aclanthology.org/D19-1053} {{M}over{S}core: Text
  generation evaluating with contextualized embeddings and earth mover
  distance}.
\newblock In \emph{Proceedings of the Conference on Empirical Methods in
  Natural Language Processing and the International Joint Conference on Natural
  Language Processing}, pages 563--578.

\bibitem[{Zhou et~al.(2021)Zhou, Gopalakrishnan, Hedayatnia, Kim, Pujara, Ren,
  Liu, and Hakkani-Tur}]{zhou-etal-2021-commonsense}
Pei Zhou, Karthik Gopalakrishnan, Behnam Hedayatnia, Seokhwan Kim, Jay Pujara,
  Xiang Ren, Yang Liu, and Dilek Hakkani-Tur. 2021.
\newblock \href {https://aclanthology.org/2021.sigdial-1.13}
  {Commonsense-focused dialogues for response generation: An empirical study}.
\newblock In \emph{Proceedings of the Annual Meeting of the Special Interest
  Group on Discourse and Dialogue}, pages 121--132.

\end{thebibliography}
\bibliographystyle{acl_natbib}
\newpage

\appendix

\onecolumn

\pagebreak

\section{Proof of Theorem~\ref{thm:decomposition}} \label{apdx:proof}

\decomposition*

\begin{proof}

\textbf{[Part (a)]} We first consider the JS decomposition. Let $p$ and $q_\theta$ be the predicted distribution for the teacher and student, respectively. Let $m(\mathbf Y) = \frac12 p(\mathbf Y) + \frac12 q_\theta(\mathbf Y)$ be their average. We claim that JS divergence between two length-$T$ sequence\footnote{In practice, $T$ can be thought of as the maximum length. Alternatively, we may consider varying-length sequences by a mixture of different values of $T$.} distributions can be decomposed step by step as
\begin{align}
    D_{\text{JS}}(p(\mathbf Y_{1:T})\|q_\theta(\mathbf Y_{1:T})) :=&\ \frac12 \E_{\mathbf Y_{1:T} \sim p} \left[ \log \frac{p(\mathbf Y_{1:T})}{m(\mathbf Y_{1:T})} \right]
    + \frac12 \E_{\mathbf Y_{1:T}' \sim q_\theta} \left[ \log \frac{q_\theta (\mathbf Y_{1:T}')}{m(\mathbf Y_{1:T}')} \right] \label{eq:JS-apdx-definition} \\ 
    =&\ \frac12  \sum_{t=1}^T \E_{\mathbf Y_{1:t-1} \sim p} \left[\sum_{\mathrm Y_t} p(\mathrm Y_t|\mathbf Y_{1:t-1}) \log \frac{p(\mathrm Y_t | \mathbf Y_{1:t-1})}{m(\mathrm Y_t | \mathbf Y_{1:t-1})}  \right] \nonumber \\
    +&\ \frac12 \sum_{t=1}^T \E_{\mathbf Y'_{1:t-1}\sim q_\theta} \left[ \sum_{\mathrm Y_t'} q_\theta(\mathrm Y_t' | \mathbf Y'_{1:t-1}) \log \frac{q_\theta(\mathrm Y'_t | \mathbf Y'_{1t-1})}{m(\mathrm Y'_t | \mathbf Y'_{1:t-1})} \right] \label{eq:JS-apdx-main}
\end{align}
For implementation, we use Monte Carlo (MC) sampling to approximate $\mathbb E_{\mathbf Y_{1:t-1}\sim p}[\cdot]$ and $\mathbb E_{\mathbf Y'_{1:t-1}\sim q_\theta}[\cdot]$, suggested by Eqn.~\eqref{eq:JS-main}. Then, we explicitly enumerate all $\mathrm Y_t$ and $\mathrm Y'_t$, because a summation over all sequences is not tractable but a step-by-step summation over words is tractable. Compared with a direct MC approximation for~\eqref{eq:JS-apdx-definition}, such step-wise decomposition allows us to propagate gradient into all the words (denoted by $\mathrm Y_t$ for the teacher and $\mathrm Y'_t$ for the student) for every step $t$.

In fact, the partially sampled sequences are reused for the summation over $t=1,\cdots, T$. That is to say, we will first sample the sequences $\mathbf y_{1:T-1}\sim p$ and $\mathbf y_{1:T-1}'\sim q_\theta$ and then compute the summation; thus, the complexity is linear rather than quadratic. 

To prove \eqref{eq:JS-apdx-main}, we first focus on the first term of \eqref{eq:JS-apdx-definition}:
\begin{align}\label{eq:js1}
    & \E_{\mathbf Y_{1:T} \sim p} \left[ \log \frac{p(\mathbf Y_{1:T})}{m(\mathbf Y_{1:T})} \right] \\ \label{eq:js2}
     =& \E_{\mathbf Y_{1:T}\sim p} \left[ \log \prod_{t=1}^T \frac{p(\mathrm Y_{t} | \mathbf Y_{1:t-1})}{m(\mathrm Y_{t} | \mathbf Y_{1:t-1})} \right] \\ \label{eq:js3}
    =& \E_{\mathbf Y_{1:T}\sim p} \Biggl[ \log \prod_{t=1}^{T-1} \frac{ p(\mathrm Y_{t} | \mathbf Y_{1:t-1})}{m(\mathrm Y_{t} | \mathbf Y_{1:t-1})} + \log \frac{p(\mathrm Y_{T} | \mathbf Y_{1:T-1})}{m(\mathrm Y_{T} | \mathbf Y_{1:T-1})} \Biggr]  \\ \label{eq:js4}
    =& \E_{\mathbf Y_{1:T}\sim p} \left[ \log \prod_{t=1}^{T-1} \frac{p(\mathrm Y_{t} | \mathbf Y_{1:t-1})}{m(\mathrm Y_{t} | \mathbf Y_{1:t-1})} \right] + \E_{\mathbf Y_{1:T}\sim p} \left[ \log \frac{p(\mathrm Y_{T} | \mathbf Y_{1:T-1})}{m(\mathrm Y_{T} | \mathbf Y_{1:T-1})} \right] \\ \label{eq:js5}
    =& \E_{\mathbf Y_{1:T-1} \sim p} \left[ \log \frac{p(\mathbf Y_{1:T-1})}{m(\mathbf Y_{1:T-1})} \right] + \E_{\mathbf Y_{1:T-1} \sim p} \left[\sum_{\mathrm Y_{T}} p(\mathrm Y_{T}|\mathbf Y_{1:T-1}) \log \frac{p(\mathrm Y_{T} | \mathbf Y_{1:T-1})}{m(\mathrm Y_{T} | \mathbf Y_{1:T-1})} \right]
\end{align}
where \eqref{eq:js2} decomposes $p(\bfY_{1:T})$ and $m(\bfY_{1:T})$; \eqref{eq:js3} and \eqref{eq:js4} split the $T$th step out. In \eqref{eq:js5}, the first term drops $\rmY_T$ because it does not occur in the expectation, and we rewrite the second term by making the summation over $\mathrm Y_t$ explicit in accordance with our sampling procedure. 

Then, we can unroll the first term of \eqref{eq:js5} recursively, resulting in
\begin{align}\label{eq:js6}
     \E_{\mathbf Y_{1:T} \sim p} \left[ \log \frac{p(\mathbf Y_{1:T})}{m(\mathbf Y_{1:T})} \right] = 
\sum_{t=1}^T \E_{\mathbf Y_{1:t-1} \sim p} \left[\sum_{\mathrm Y_t} p(\mathrm Y_t|\mathbf Y_{1:t-1}) \log \frac{p(\mathrm Y_t | \mathbf Y_{1:t-1})}{m(\mathrm Y_t | \mathbf Y_{1:t-1})}  \right]\end{align}

Likewise, the term $\E_{\mathbf Y_{1:T}' \sim q_\theta} \left[ \log \frac{q_\theta (\mathbf Y_{1:T}')}{m(\mathbf Y_{1:T}')} \right]$ in \eqref{eq:JS-apdx-definition} is treated in a similar fashion, concluding our proof for JS decomposition.

We state KL and RKL decompositions below. Their proofs are similar and thus omitted.

\begin{align}
    D_{\text{KL}}(p(\mathbf Y_{1:T})\|q_\theta(\mathbf Y_{1:T})) =& \sum_{t=1}^T \E_{\mathbf Y_{1:t-1} \sim p} \left[\sum_{\mathrm Y_t} p(\mathrm Y_t|\mathbf Y_{1:t-1}) \log \frac{p(\mathrm Y_t | \mathbf Y_{1:t-1})}{q_\theta(\mathrm Y_t | \mathbf Y_{1:t-1})}  \right] \\
    D_{\text{RKL}}(p(\mathbf Y_{1:T})\|q_\theta(\mathbf Y_{1:T})) =& \sum_{t=1}^T \E_{\mathbf Y'_{1:t-1}\sim q_\theta} \left[ \sum_{\mathrm Y_t'} q_\theta(\mathrm Y_t' | \mathbf Y_{1:t-1}') \log \frac{q_\theta(\mathrm Y'_t | \mathbf Y'_{1:t-1})}{p(\mathrm Y'_t | \mathbf Y'_{1:t-1})} \right] 
\end{align}

\bigskip
\textbf{[Part (b)]} This part shows that the same step-wise decomposition for TVD is an upper bound:
\begin{align}
D_{\text{TVD}}(p(\mathbf Y_{1:T}) \| q_\theta(\mathbf Y_{1:T})) :=& \frac12 \sum_{\mathbf Y_{1:T}} |q_\theta(\mathbf Y_{1:T}) - p(\mathbf Y_{1:T})| \\
\le& \frac12 \left[\rule{0em}{9mm}\right.
 \frac12 \sum_{t=1}^T \E_{\mathbf Y_{1:t-1} \sim p} \left[ \sum_{\mathrm Y_t} \Big| q_\theta (\mathrm Y_t | \mathbf Y_{1:t-1}) - p(\mathrm Y_t | \mathbf Y_{1:t-1}) \Big| \right] \\
 &+\frac12 \sum_{t=1}^T \E_{\mathbf Y'_{1:t-1} \sim q_\theta} \Bigg[ \sum_{\mathrm Y'_t} \Big| q_\theta (\mathrm Y'_t | \mathbf Y'_{1:t-1}) - p(\mathrm Y'_t | \mathbf Y'_{1:t-1}) \Big| \Bigg] \left.\rule{0em}{9mm}\right]
 \label{eq:thm-tvd}
\end{align}

We again start by re-writing the TVD loss in a recursive form

\begin{align}
    & D_{\text{TVD}}(p(\mathbf 
 Y_{1:T}) \| q_\theta(\mathbf Y_{1:T}))  = \frac12 \sum_{\mathbf Y_{1:T}} |q_\theta(\mathbf Y_{1:T}) - p(\mathbf Y_{1:T})| \label{eq:tvd1} \\
    & = \frac12 \sum_{\mathbf Y_{1:T-1}} \sum_{\mathrm Y_T} |q_\theta(\mathbf Y_{1:T-1}) q_\theta (\mathrm Y_T | \mathbf Y_{1:T-1}) - p(\mathbf Y_{1:T-1}) p(\mathrm Y_T | \mathbf Y_{1:T-1})| \label{eq:tvd2} \\
    & = \frac12 \sum_{\mathbf Y_{1:T-1}} \sum_{\mathrm Y_T} \frac{p(\mathbf Y_{1:T-1})}{p(\mathbf Y_{1:T-1})} |q_\theta(\mathbf Y_{1:T-1}) q_\theta (\mathrm Y_T | \mathbf Y_{1:T-1}) - p(\mathbf Y_{1:T-1}) p(\mathrm Y_T | \mathbf Y_{1:T-1})| \label{eq:tvd3} \\
    & = \frac12 \sum_{\mathbf Y_{1:T-1}} p(\mathbf Y_{1:T-1}) \sum_{\mathrm Y_T} \left| \frac{q_\theta(\mathbf Y_{1:T-1}) q_\theta (\mathrm Y_T | \mathbf Y_{1:T-1})}{p (\mathbf Y_{1:T-1})} - p(\mathrm Y_T | \mathbf Y_{1:T-1}) \right| \label{eq:tvd4} \\
    & = \scalemath{0.92} { \frac12 \E_{\mathbf Y_{1:T-1}\sim p} \left[ \sum_{\mathrm Y_T} \Big| \frac{q_\theta(\mathbf Y_{1:T-1}) q_\theta (\mathrm Y_T | \mathbf Y_{1:T-1})}{p (\mathbf Y_{1:T-1})} - q_\theta(\mathrm Y_T | \mathbf Y_{1:T-1}) + q_\theta(\mathrm Y_T | \mathbf Y_{1:T-1}) - p(\mathrm Y_T | \mathbf Y_{1:T-1}) \Big| \right] \label{eq:tvd5} } \\
    & = \scalemath{0.92} { \frac12 \E_{\mathbf Y_{1:T-1}\sim p} \left[ \sum_{\mathrm Y_T} \Big| \frac{q_\theta(\mathrm Y_T | \mathbf Y_{1:T-1})}{p(\mathbf Y_{1:T-1})} \big( q_\theta(\mathbf Y_{1:T-1}) - p(\mathbf Y_{1:T-1}) \big)  + q_\theta(\mathrm Y_T | \mathbf Y_{1:T-1}) - p(\mathrm Y_T | \mathbf Y_{1:T-1}) \Big| \right] \label{eq:tvd6} } \\
    &\leq \scalemath{0.92}{ \frac12 \E_{\mathbf Y_{1:T-1}\sim p} \left[ \sum_{\mathrm Y_T}\left( \Big| \frac{q_\theta(\mathrm Y_T | \mathbf Y_{1:T-1})}{p(\mathbf Y_{1:T-1})} \big( q_\theta(\mathbf Y_{1:T-1}) - p(\mathbf Y_{1:T-1}) \big) \Big| + \Big| q_\theta(\mathrm Y_T | \mathbf Y_{1:T-1}) - p(\mathrm Y_T | \mathbf Y_{1:T-1}) \Big| \right)\right] \label{eq:tvd7} } \\
    &= \frac12 \E_{\mathbf Y_{1:T-1} \sim p} \left[ \sum_{\mathrm Y_T} \left| \frac{q_\theta(\mathrm Y_T | \mathbf Y_{1:T-1})}{p(\mathbf Y_{1:T-1})} \big( q_\theta(\mathbf Y_{1:T-1}) - p(\mathbf Y_{1:T-1}) \big) \right| \right] + \nonumber \\
    & \quad\,\, \frac12 \E_{\mathbf Y_{1:T-1} \sim p} \left[ \sum_{\mathrm Y_T} \left| q_\theta(\mathrm Y_T | \mathbf Y_{1:T-1}) - p(\mathrm Y_T | \mathbf Y_{1:T-1}) \right|  \right] \label{eq:tvd8} \\
    &= \frac12 \E_{\mathbf Y_{1:T-1} \sim p} \left[ \left| \frac{1}{p(\mathbf Y_{1:T-1})} \big( q_\theta(\mathbf Y_{1:T-1}) - p(\mathbf Y_{1:T-1}) \big) \right| \right] + \nonumber \\
    & \quad\,\, \frac12 \E_{\mathbf Y_{1:T-1} \sim p} \left[ \sum_{\mathrm Y_T} \left| q_\theta(\mathrm Y_T | \mathbf Y_{1:T-1}) - p(\mathrm Y_T | \mathbf Y_{1:T-1}) \right|  \right] \label{eq:tvd9} \\
    &= \frac12 \sum_{\mathbf Y_{1:T-1}} \left| q_\theta(\mathbf Y_{1:T-1}) - p(\mathbf Y_{1:T-1}) \right| + \frac12 \E_{\mathbf Y_{1:T-1} \sim p} \left[ \sum_{\mathrm Y_T} \left| q_\theta (\mathrm Y_T | \mathbf Y_{1:T-1}) - p(\mathrm Y_T | \mathbf Y_{1:T-1}) \right| \right] \label{eq:tvd10} \\
    &= \frac12 \sum_{t=1}^T \E_{\mathbf Y_{1:t-1} \sim p} \left[ \sum_{\mathrm Y_t} \left| q_\theta (\mathrm Y_t | \mathbf Y_{1:t-1}) - p(\mathrm Y_t | \mathbf Y_{<1:t}) \right| \right] \label{eq:tvd11}
\end{align}
where \eqref{eq:tvd2} breaks the sequence-level summation into the first $T-1$ steps and the last step,
\eqref{eq:tvd3} multiplies and divides $p({\bfY_{1:T-1}})$, and
\eqref{eq:tvd5} subtracts and adds $q_\theta(\rmY_T | \bfY_{1:T-1})$.
After some regrouping in \eqref{eq:tvd6}, we apply the triangle inequality in \eqref{eq:tvd7}.
In \eqref{eq:tvd8}, we break the expectation into two terms, where the first term is further simplified by summing over $\mathrm Y_T$ in \eqref{eq:tvd9} and expanding the expectation in \eqref{eq:tvd10}.
These manipulations bring the equation to a recursive form. By applying the same technique as in \eqref{eq:js6}, we may further unroll the first term in \eqref{eq:tvd10} and eventually obtain  \eqref{eq:tvd11} as an upper bound.

Likewise, we can obtain the following inequality by multiplying and dividing by $q_\theta(\mathbf y_{1:T-1})$ in \eqref{eq:tvd3}
\begin{align}
     & \mathcal L_{\text{TVD}} \le \frac12 \sum_{t=1}^T \E_{\mathbf Y'_{1:t-1} \sim q_\theta} \left[ \sum_{\mathrm Y'_t} \left| q_\theta (\mathrm Y'_t | \mathbf Y'_{1:t-1}) - p(\mathrm Y'_t | \mathbf Y'_{1:t-1}) \right| \right] \label{eq:tvd12}
\end{align}
These two upper bounds, \eqref{eq:tvd11} and \eqref{eq:tvd12}, are then combined to obtain \eqref{eq:thm-tvd}, concluding the proof.

\end{proof}

Admittedly, both \eqref{eq:tvd11} and \eqref{eq:tvd12} are valid upper bounds for the TVD divergence, but we nevertheless combine these two formulas to obtain a more computationally robust upper bound in the same spirit of JS decomposition.

\begin{table}[!b]
\centering
\resizebox{0.80\columnwidth}{!}{
\begin{tabular}{|l|l|rrr|}
\hline
\multicolumn{1}{|c|}{\multirow{2}{*}{Dataset}} &
  \multicolumn{1}{c|}{\multirow{2}{*}{Task}} &
  \multicolumn{3}{c|}{\# of Samples} \\ \cline{3-5} 
\multicolumn{1}{|c|}{} &
  \multicolumn{1}{c|}{} &
  \multicolumn{1}{c}{Train} &
  \multicolumn{1}{c}{Dev} &
  \multicolumn{1}{c|}{Test} \\ \hline
DART~\citep{nan-etal-2021-dart} &
  Data-to-Text Generation &
  30,526 &
  2,768 &
  4,159 \\
XSum~\citep{narayan-etal-2018-xsum} &
  Summarization &
  204,045 &
  11,332 &
  11,334 \\
WNT16 EN-RO~\citep{bojar-etal-2016-wmt16} &
  Machine Translation &
  100,000 &
  1,999 &
  1,999 \\
Commonsense Dialogue~\citep{zhou-etal-2021-commonsense} &
  Dialogue Generation &
  51,831 &
  6,619 &
  6,610 \\ \hline
\end{tabular}
}
\caption{Statistics of our datasets.}

\label{tab:exp-details}
\end{table}

\section{Experimental Details} \label{apdx:exp-details}

Table~\ref{tab:exp-details} shows the statistics of our datasets.
As seen, we benchmarked our models on a variety of natural language generation tasks with different data sizes. We chose state-of-the-art models as the teachers, with 200M--400M parameters. Accordingly, our students had 50M--150M parameters. 
The high performance of \fdistill\ variants across these datasets highlights the robustness of our approach.

For training, we used the Adam optimizer~\citep{adam-optimizer} with  default hyperparameters $\beta=(0.9, 0.999)$ on DART, XSum, and Commonsense Dialogue.
For WMT16 EN-RO, we followed the T5 teacher model~\citep{t52020} and used the AdaFactor optimizer~\citep{pmlr-v80-adafactor}. 
We chose a small batch size of eight to fit  the student as well as the large teacher in our GPU. 
All student models were trained for 28 epochs for pre-distillation and another 12 epochs for each distilling method, as additional training did not further improve performance.

The main experiments were conducted on AMD Milan 7413 CPUs and NVidia A100 GPUs, and the total training time was estimated at 2000 GPU hours. Note that this is not because our algorithm is slow (efficiency analyzed in Table~\ref{tab:efficiency}), but because we have extensively experimented with a variety of datasets and model variants.

\section{Case Study} \label{apdx:case-study}
Table~\ref{tab:case} shows example outputs for DART and Commonsense Dialogue.
On the DART dataset, the KL and RKL distillations fail to yield coherent responses from the input data records.
By contrast, JS and TVD distillations enable the student to generate sentences of much higher quality: they correctly recognize the name of the train as well as its origin and destination.

We additionally show an example output from the Commonsense Dialogue dataset, because the dialogue task exhibits the most severe multi-modal problem, which in turn requires the student to carefully balance mode averaging and collapsing.
As seen, the KL-distilled student generates a short and generic response, which is consistent with existing literature~\citep{wei2019neural, bao-etal-2020-plato}, explained as  mode averaging in our paper. The RKL-distilled student generates a detailed, but ungrammatical and incoherent, response.
For JS and TVD distillations, the students generate responses that are both coherent and detailed. 
The case studies confirm our main claim that JS and JVD are more effective sequence-level distilling approaches.

\begin{table}[!t]
\centering
\resizebox{\columnwidth}{!}{
\begin{tabular}{|l|l|}
\hline
\multicolumn{2}{|c|}{\textbf{DART}} \\ \hline
Input &  (11028/11027, destination, mumbai), (11028/11027, origin, Chennai), (11028/11027, train\_name, mumbai Mail)  \\
Reference & train no . 11028 / 11027 name of the train is mumbai mail has origin and destination of chennai to mumbai \\ \hline
KL & the train name of the train that originates from chennai is mumbai.                                             \\
RKL & the mumbai mumbai mumbai train name is mumbai mumbai and the destination is mumbai.                            \\
JS & the mumbai mail train goes to mumbai and originates from chennai to mumbai.                                     \\
TVD & the mumbai mail train starts from chennai to mumbai.   \\ \hline \hline
\multicolumn{2}{|c|}{\textbf{Commonsense Dialogue}} \\ \hline
Input & Quinn spent many years studying. Finally it became graduation time for him. \\
Reference & I can't describe how happy I am on this day.   \\ \hline
KL & I am done. Now is over. \\
RKL & I am going to miss my mom and dad. I am going to miss my dad and the people I have to do a lot of things.  \\
JS  & I am so excited for my graduation. I can't wait to get back to my home and spend some time with my family. \\
TVD & I am finally done packing my dorm. I can't wait to start my first year of college. \\ \hline
\end{tabular}
}
\caption{Example outputs from \fdistill\ variants.}
\label{tab:case}
\end{table}

\end{document}